\documentclass[journal,final,twocolumn]{IEEEtran}
\usepackage{cite}
\IEEEpubid{0000--0000/00\$00.00˜\copyright˜2018 IEEE}

\usepackage{amssymb} 
\usepackage{amsmath}
\usepackage{graphicx}
\usepackage{algorithm}
\usepackage{algorithmicx}
\usepackage{algpseudocode}
\usepackage{mathrsfs}
\usepackage{multirow}
\usepackage{amsthm}
\usepackage{bm}

\usepackage{caption2}

\newtheorem{theorem}{Theorem}
\newtheorem{lemma}{Lemma}
\newtheorem{assumption}{Assumption}

\begin{document}

\title{\textbf{Exactly Robust Kernel Principal Component Analysis}}
\author{Jicong Fan, Tommy W.S. Chow,~\IEEEmembership{Fellow,~IEEE}
\IEEEcompsocitemizethanks{\IEEEcompsocthanksitem The authors are with the Department of Electronic Engineering, City University of Hong Kong, Kowloon, Hong Kong SAR, China. \protect\\
Corresponding author: Tommy W.S. Chow. E-mail: eetchow@cityu.edu.hk

$\copyright$ 2019 IEEE.  Personal use of this material is permitted.  Permission from IEEE must be obtained for all other uses, in any current or future media, including reprinting/republishing this material for advertising or promotional purposes, creating new collective works, for resale or redistribution to servers or lists, or reuse of any copyrighted component of this work in other works.}
}

%\date{}

\IEEEtitleabstractindextext{%
\begin{abstract}
Robust principal component analysis (RPCA) can recover low-rank matrices when they are corrupted by sparse noises. In practice, many matrices are, however, of high-rank and hence cannot be recovered by RPCA. We propose a novel method called robust kernel principal component analysis (RKPCA) to decompose a partially corrupted matrix as a sparse matrix plus a high or full-rank matrix with low latent dimensionality. RKPCA can be applied to many problems such as noise removal and subspace clustering and is still the only unsupervised nonlinear method robust to sparse noises. Our theoretical analysis shows that, with high probability, RKPCA can provide high recovery accuracy. The optimization of RKPCA involves nonconvex and indifferentiable problems. We propose two nonconvex optimization algorithms for RKPCA. They are alternating direction method of multipliers with backtracking line search and proximal linearized minimization with adaptive step size. Comparative studies in noise removal and robust subspace clustering corroborate the effectiveness and superiority of RKPCA.
\end{abstract}

}

\maketitle
\IEEEdisplaynontitleabstractindextext
\section{Introduction}
\IEEEPARstart{P}{rincipal} component analysis (PCA) \cite{PCA_Jolliffe2002} is a well-known and powerful technique that has been widely used in many areas such as computer science, economy, biology, and chemistry. In these areas, the data are often redundant, which means they can often be represented by reduced number of features. PCA finds a set of orthogonal projections to transform high-dimensional data or observed variables into low-dimensional latent variables with the reconstruction errors being minimized. The orthogonal projections can be found by singular value decomposition or eigenvalue decomposition. PCA is widely used for dimensionality reduction, feature extraction, and noise removal. 

PCA is a linear method that is not effective in handling nonlinear data. To solve nonlinear problem, kernel PCA (KPCA) \cite{KPCA1998} was derived. In KPCA, observed data are transformed by a nonlinear function into a high (possibly infinite) dimensional feature space using kernel trick. The latent variables can be extracted from the feature space without explicitly carrying out the nonlinear mapping. KPCA has been applied to many practical problems such as dimensionality reduction, novelty detection \cite{HOFFMANN2007863,FAN2014205}, and image denoising \cite{Mika99kernelpca}. KPCA, however, cannot be directly applied to denoising problems. In \cite{Mika99kernelpca,PreImageKPCA2004TNN,PreImageKPCA2010TNN}, the pre-image problem of KPCA was studied. In the problem, the principal components obtained from the feature space are mapped back into the data space to reconstruct the observed variables, where the reconstruction errors are regarded as noises. In \cite{nguyen2009robust}, a robust KPCA was proposed to handle noise, missing data, and outliers in KPCA.
\IEEEpubidadjcol

Another limitation of PCA is that it is not robust to sparse corruptions and outliers. Thus, several robust PCA methods were proposed \cite{xu1995robust,RPCA2001,Ding:2006:RPR:1143844.1143880,RPCA}. In \cite{RPCA2001}, through using the Geman-McClure function $f(x,\sigma)=x^2/(x^2+\sigma^2)$, a robust PCA (RPCA) was developed for computer vision problems. In \cite{RPCA}, based on nuclear norm and $\ell_1$ norm minimizations, another RPCA was proposed to decompose a noisy matrix into a low-rank matrix plus a sparse matrix. RPCA of \cite{RPCA} is able to outperform that of \cite{RPCA2001} both theoretically and experimentally. The low-rank or/and sparse models have been widely studied and exploited in many problems such as matrix completion \cite{CandesRecht2009,TIP_MC_2016,Fan2017290,FAN201834,8469061} and subspace clustering \cite{SSC2009,LRR_PAMI_2013,SSC_PAMIN_2013,xiao2016robust,NLRR_2016,Fan201736,fan2018accelerated}. In subspace clustering, RPCA could outperform low-rank representation (LRR) \cite{LRR_PAMI_2013} and sparse representation (SSC) \cite{SSC_PAMIN_2013} when the data are heavily corrupted by sparse noises or/and outliers, because the dictionary used in LRR and SSC is the data matrix itself which would introduce considerable bias into the representations. If the data are pre-processed by RPCA, the clustering accuracy can be significantly improved. A few recent extensions of RPCA can be found in \cite{feng2013online,hauberg2014grassmann,zhao2014robust,pimentel2017random,7878661}.

KPCA is not robust to sparse noises \cite{nguyen2009robust} and RPCA and its recent extensions (because of the linear model and low-rank assumption \cite{RPCA,7878661}) are unable to handle nonlinear data and high-rank matrices. Therefore, there is a need to derive a variant of PCA that is able to handle nonliear data and is also robust to sparse noises. It is worth noting that the robust KPCA proposed in \cite{nguyen2009robust} is a supervised method for handling missing data and sparse noises. The method first requires to learn a KPCA model on a clean training dataset. It utilizes the model to handle new dataset with missing data or intra-sample outliers. Although the authors in \cite{nguyen2009robust} mentioned that their robust KPCA could deal with intra-sample outliers in training data through modifying the Algorithm 1 of their paper, there were no algorithmic details and experimental results. Similar to \cite{RPCA2001}, the robustness of \cite{nguyen2009robust} to intra-sample outliers was obtained using the Geman-McClure function. However, the Geman-McClure function is inferior to $\ell_1$ norm in terms of efficiency, accuracy, and theoretical guarantee under broad conditions \cite{RPCA}.

In this paper, we propose a novel method called robust kernel principal component analysis (RKPCA) to handle nonlinear data and sparse noises simultaneously. RKPCA assumes that the nonlinear transformation of the data matrix is of low-rank, which is different from the low-rank assumption of the data matrix itself in RPCA. RKPCA is effective in recovering high-rank or even full-rank matrices and robust to sparse noises and outliers. In this paper, we also provide theoretical support for RKPCA. The optimization of RKPCA is challenging because it involves nonconvex and indifferentiable problems at the same time. We propose nonconvex alternating direction method of multipliers with backtracking line search and nonconvex proximal linearized minimization with adaptive step size for RKPCA. Thorough comparative studies were conducted on synthetic data, nature images, and motion data. The experimental results verify that: (1) RKPCA is more effective than PCA, KPCA, and RPCA in recognizing sparse noises when the data have nonlinear structures; (2) RKPCA is more robust and effective than RPCA, SSC, NLRR \cite{NLRR_2016}, and RKLRS \cite{xiao2016robust} in subspace clustering.

\section{Methodology}
\subsection{Nonlinear model and RKPCA solution}
Suppose that an observed data matrix $\bm{M}\in\mathbb{R}^{d\times n}$ is given by
\begin{equation}\label{Eq.MXE}
\bm{M}=\bm{X}+\bm{E},
\end{equation}
where $\bm{X}$ is the matrix of clean data and $\bm{E}$ is the matrix of sparse noises (randomly distributed). Our goal is to recover $\bm{X}$ and $\bm{E}$. In RPCA, $\bm{X}$ is assumed to be of low-rank. Hence, RPCA aims at solving the following problem
\begin{equation}\label{Eq.RPCA_0}
\min\limits_{X,E}\textup{rank}(\bm{X})+\lambda\Vert \bm{E}\Vert_0,\ s.t. \ \bm{X}+\bm{E}=\bm{M},
\end{equation}
where $\textup{rank}(\bm{X})$ denotes the rank of $\bm{X}$, $\Vert \bm{E}\Vert_0$ is the $\ell_0$ norm of $\bm{E}$ defined by the number of non-zero elements in $\bm{E}$, and $\lambda$ is a parameter to balance the two terms. Both rank minimization and $\ell_0$ norm minimization are NP-hard. Therefore, problem (\ref{Eq.RPCA_0}) is approximated as
\begin{equation}\label{Eq.RPCA}
\min\limits_{\bm{X},\bm{E}}\Vert \bm{X}\Vert_\ast+\lambda\Vert \bm{E}\Vert_1,\ s.t. \ \bm{X}+\bm{E}=\bm{M},
\end{equation}
where $\Vert \bm{X}\Vert_\ast$ denotes the nuclear norm of $\bm{X}$ and $\Vert \bm{E}\Vert_1$ denotes the $\ell_1$ norm of $\bm{E}$. Nuclear norm, defined by the sum of singular values, is a convex relaxation of matrix rank. $\ell_1$ norm is a convex relaxation of $\ell_0$ norm. Nuclear norm minimization and $\ell_1$ norm minimization can be solved via singular value thresholding and soft thresholding respectively.

The low-rank assumption made in RPCA indicates that $\bm{X}$ is given by a low-dimensional linear latent variable model, i.e., $\bm{X}=\bm{P}\bm{Z}$, where $\bm{P}\in\mathbb{R}^{d\times r}$ is the projection matrix, $\bm{Z}\in\mathbb{R}^{r\times n}$ consists of the latent variables, and $r$ is the rank of $\bm{X}$. In practice, the columns of $\bm{X}$ may be drawn from low-dimensional nonlinear latent variable models. Specifically, in this paper, we make the following assumption.
\begin{assumption}\label{assump_0}
The columns of $\bm{X}\in\mathbb{R}^{d\times m}$ are given by 
\begin{equation}\label{Eq.nllvm}
\bm{x}=f(\bm{z}), 
\end{equation}
where $\bm{z}\in\mathbb{R}^r$ consists of uncorrelated latent variables,  $f:\mathbb{R}^r\rightarrow \mathbb{R}^d$ is a nonlinear smooth mapping, and $r\ll d$. In addition, the dimension of the manifold defined by $f$ is $r$. 
\end{assumption}
For convenience, we denote $\bm{X}=f(\bm{Z})$, where $f$ is performed on each column of $\bm{Z}\in\mathbb{R}^{r\times n}$ separately and $n>d$. Although $r$ is much smaller than $d$, $\bm{X}$ could be of high-rank or even full-rank, which is beyond the assumption of RPCA. Therefore $\bm{X}$ and $\bm{E}$ cannot be recovered by RPCA. We denote the latent dimensionality of $\bm{X}$ by $\textup{ldim}(\bm{X})$ and have $\textup{ldim}(\bm{X})\leq \textup{rank}(\bm{X})$, where the equality holds when $f$ is linear. 

For example, we uniformly draw 100 samples of a single variable $z$ from the interval $[-1,1]$ and generate a $3\times 100$ matrix by $\bm{x}=f(z)=[z,z^2,z^3]^T$. We add Gaussian noise $e\sim\mathcal{N}(0,0.1)$ to one randomly chosen entry of each column of the matrix. The clean data and corrupted data are shown in Figure \ref{Fig.syn_curve}(a). Figure \ref{Fig.syn_curve}(b) shows the data recovered by RPCA while Figure \ref{Fig.syn_curve}(c) shows the data recovered by the proposed method in this paper. Since the matrix is of full-rank though the latent dimension is $1$, RPCA failed in recovering the data. In contrast, our proposed method has a good performance. 
\begin{figure}[h!]
\centering
\includegraphics[width=8.5cm]{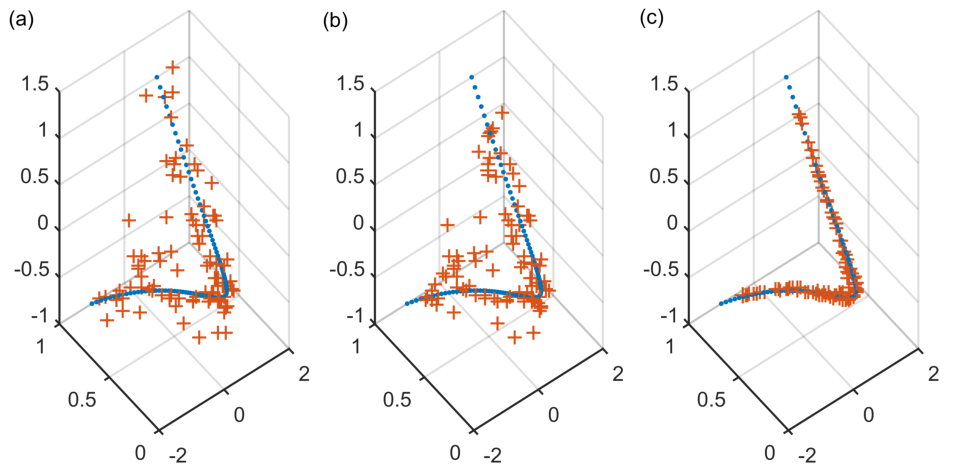}
\caption{An example of 3D data of one latent dimension: (a) data with sparse noise; (b) data recovered by RPCA; (c) data recovered by the proposed method. The clean data are shown by the blue points.}\label{Fig.syn_curve}
\end{figure}

To make the recovery problem meaningful, we propose the following assumption: 
\begin{assumption}\label{asp_spread}
(\textit{non-sparse component condition}) Given $f(\cdot):\mathbb{R}^r\rightarrow \mathbb{R}^d$ and $\bm{z}=[z_1,z_2,\cdots,z_r]^T$, for any decomposition (if exists) $f(\bm{z})=g(\bm{z_{|i}})+h(z_i)$, where $\bm{z}_{|i}=[z_1,\cdots,z_{i-1},z_{i+1},\cdots,z_r]^T$ and $h(z_i)=[h_1(z_i),h_2(z_i),\cdots,h_d(z_i)]^T$, one has $\textup{Pr}(h_j(z_i)\neq 0)=\mu_0$ $\forall$ $j=1,\dots,d$. 
\end{assumption}
Assumption \ref{asp_spread} indicates that, if $\bm{X}=\bm{X}_a+\bm{X}_b$ with $\textup{ldim}(\bm{X}_a)=r-1$ and $\textup{ldim}(\bm{X}_b)=1$, then $\Vert \bm{X}_b\Vert_0=\mu_0 dn$. Suppose $\bm{E}=\bm{E}_a+\bm{E}_b$ with $\textup{ldim}(\bm{E}_b)=1$ and $\textup{ldim}(\bm{X}-\bm{X}_b+\bm{E}_b)=r$, we should have $\Vert\bm{E}-\bm{E}_b+\bm{X}_b\Vert_0>\vert\bm{E}\vert_0$. It requires that $\Vert\bm{X}_b\Vert_0>\Vert\bm{E}_b\Vert_0$. Since $\Vert\bm{E}_b\Vert_0$ is at most $\Vert\bm{E}\Vert_0/d$, we have $\mu_0 dn>\Vert\bm{E}\Vert_0/d$. Denoting $\rho=\Vert \bm{E}\Vert_0/(dn)$ as the noise density, we have
\begin{equation}\label{Eq.rho_mu}
\rho<\mu_0 d,
\end{equation}
which always holds if $\mu_0>1/d$. It is found that when one component of $\bm{X}$ is highly sparse, the noise density of $\bm{E}$ should be low enough. With condition (\ref{Eq.rho_mu}), we give the following lemma:
\begin{lemma}\label{lem_rec}
$\bm{X}$ and $\bm{E}$ can be recovered if the following problem is solved:
\begin{equation}\label{Eq.minldim}
\min\limits_{\bm{X},\bm{E}}\textup{ldim}(\bm{X})+\lambda\Vert \bm{E}\Vert_0,\ s.t. \ \bm{X}+\bm{E}=\bm{M},
\end{equation}
where $r/(dn-\rho dn)<\lambda<(d-r)/(\rho dn)$.
\end{lemma}
\begin{proof}
Denote the optimal value of $\bm{X}$ and $\bm{E}$ by $\bm{X}^*$ and $\bm{E}^*$. Denote $J(\bm{X},\bm{E})=\textup{ldim}(\bm{X})+\lambda\Vert \bm{E}\Vert_0$. Then $J(\bm{X}^*,\bm{E}^*)=r+\lambda \rho dn$. Given an arbitrary $\bm{X}'=\bm{X}^*+\bm{\Delta}$ with $\bm{\Delta} \neq \bm{0}$, we denote $\bm{E}'=\bm{E}-\bm{\Delta}$. If $\textup{ldim}(\bm{X}')<r$, we have $J(\bm{X}',\bm{E}')_{min}=\lambda nd>J(\bm{X}^*,\bm{E}^*)$ provided that $r<\lambda (dn-\rho dn)$. If $\textup{ldim}(\bm{X}')<r$, $J(\bm{X}',\bm{E}')_{min}=d>J(\bm{X}^*,\bm{E}^*)$ provided that $\lambda<(d-r)/(\rho dn)$. $J(\bm{X}',\bm{E}')=J(\bm{X}^*,\bm{E}^*)$ only if $\bm{\Delta}=0$.
\vspace{-0.1cm}
\end{proof}
Though (\ref{Eq.minldim}) is intractable,  Lemma \ref{lem_rec} indicates the feasibility that $\bm{X}$ can be recovered when $f(\cdot)$ is nonlinear. If we can find a tractable relaxation of (\ref{Eq.minldim}) (especially for $\textup{ldim}(\bm{X})$), we can obtain $\bm{X}$ and $\bm{E}$.

In this paper, we give the following theorem:
\begin{theorem}\label{asp_kernel}
Suppose $\bm{X}\in\mathbb{R}^{d\times n}$ is given by Assumption \ref{assump_0} and denote $\phi(\bm{X})=[\phi(\bm{x}_1),\phi(\bm{x}_2),\cdots,\phi(\bm{x}_n)]$. Then there exists a smooth nonlinear function $\phi:\mathbb{R}^d\rightarrow \mathcal{F}^\kappa$ that maps the columns of $\bm{X}$ into a high-dimensional (possibly infinite) feature space such that $\phi(\bm{X})\in\mathbb{R}^{\kappa\times n}$ is exactly or approximately of low-rank, provided that $n$ is sufficiently large.
\end{theorem}
\begin{proof}
We denote $\phi(\bm{x})=\phi(f(\bm{z}))\triangleq\psi(\bm{z})$. Then $\psi:\mathbb{R}^r\rightarrow \mathcal{F}^\kappa$ is a smooth nonlinear mapping and hence has Taylor expansion convergent at least locally. The Taylor expansion of $\psi(\bm{z})$ at $\bm{\upsilon}$ is
\begin{equation}
\psi(\bm{z})=\sum_{\vert \alpha \vert\geq 0}\dfrac{(\bm{z}-\bm{\upsilon})^{\alpha}}{\alpha !}(\partial^{\alpha} \phi)(\bm{\upsilon}),
\end{equation}
where we have used the multi-index notation.
The $t$-th order Taylor approximation of $\psi(\bm{z})$ is
\begin{equation}
\psi(\bm{z})\approx \bm{\Theta}_0+\bm{\Theta}_1\tilde{\bm{z}}^{(1)}+\cdots+\bm{\Theta}_t\tilde{\bm{z}}^{(t)},
\end{equation}
where $\bm{\Theta}_0\in\mathbb{R}^{\kappa\times 1}$. For $k=1,\cdots,t$, $\bm{\Theta}_k\in\mathbb{R}^{\kappa\times d_k}$, $\tilde{\bm{z}}^{(k)}\in\mathbb{R}^{d_k\times 1}$ consists of the $k$-combinations (product) of the elements in $\bm{z}-\bm{\upsilon}$, and $d_k=\binom{r+k-1}{k}$. Denote $\psi(\bm{Z})=[\psi(\bm{z}_1),\psi(\bm{z}_2),\cdots,\psi(\bm{z}_n)]$, $\bm{\Theta}=[\bm{\Theta}_0,\bm{\Theta}_1,\cdots,\bm{\Theta}_t]$, and $\bar{\bm{Z}}=[\bar{\bm{z}}_1,\bar{\bm{z}}_2,\cdots,\bar{\bm{z}}_n]$, where $\bar{\bm{z}}=[1;\tilde{\bm{z}}^{(1)};\cdots;\tilde{\bm{z}}^{(t)}]$. Then we have $\bm{\Theta}\in\mathbb{R}^{\kappa\times \bar{r}}$, $\bar{\bm{Z}}\in\mathbb{R}^{\bar{r}\times n}$ and $\psi(\bm{Z})\approx\bm{\Theta}\bar{\bm{Z}}$, where $\bar{r}=1+\sum_{k=1}^t\binom{r+k-1}{k}=\binom{r+t}{t}$. It means
\begin{equation}\label{Eq.rank_phiX}
\textup{rank}(\phi(\bm{X}))=\textup{rank}(\psi(\bm{Z}))\approx\binom{r+t}{t},
\end{equation}
provided that $n$ and $\kappa$ are large. Therefore, $\phi(\bm{X})$ is exactly of low-rank when the Taylor residuals vanish and approximately of low-rank otherwise. For example, when $d=10$, $r=2$, $n=100$, and $t=4$, we have $\textup{rank}(\phi(\bm{X}))\approx15\ll 100$, although $\bm{X}$ could be of full-rank.
\end{proof}
Suppose that $\bm{X}$ is corrupted by sparse noises, e.g., $\hat{\bm{X}}=\bm{X}+\bm{E}$ and $\bm{E}\neq \bm{0}$, the latent dimension of $\hat{\bm{X}}$ will be higher than that of $\bm{X}$, e.g. $r'>r$. According to (\ref{Eq.rank_phiX}), we have
\begin{equation}\label{Eq.rxrx}
\textup{rank}(\phi(\hat{\bm{X}}))>\textup{rank}(\phi(\bm{X})).
\end{equation}
In fact, (\ref{Eq.rank_phiX}) indicates that the rank of $\phi(\bm{X})$ is a monotone increasing function of the latent dimension of $\bm{X}$. Therefore, minimizing the rank of $\phi(\bm{X})$ amounts to minimizing the latent dimension of $\bm{X}$. Now problem (\ref{Eq.minldim}) can be rewritten as
\begin{equation}\label{Eq.RKPCA_0}
\min\limits_{\bm{X},\bm{E}}\textup{rank}(\phi(\bm{X}))+\lambda\Vert \bm{E}\Vert_0,\ s.t. \ \bm{X}+\bm{E}=\bm{M},
\end{equation}
and further approximated by
\begin{equation}\label{Eq.RKPCA_0.5}
\min\limits_{\bm{X},\bm{E}}\Vert \phi(\bm{X})\Vert_\ast+\lambda\Vert \bm{E}\Vert_1,\ s.t. \ \bm{X}+\bm{E}=\bm{M},
\end{equation}
where $\Vert \phi(\bm{X})\Vert_\ast=\sum_{i=i}^{min(\kappa,n)}\sigma_i$ and $\sigma_i$ is the $i$th singular value of $\phi(\bm{X})$. We also have
\begin{equation}\label{Eq.SpNormT}
\Vert \phi(\bm{X})\Vert_\ast=\textup{Tr}((\phi(\bm{X})^T\phi(\bm{X}))^{1/2}),
\end{equation}
where $\textup{Tr}(\cdot)$ denotes matrix trace. Denoting the SVD of $\phi(\bm{X})$ by $\phi(\bm{X})=\bm{U}\bm{S}\bm{V}^T$, equation (\ref{Eq.SpNormT}) can be proved as following:
\begin{equation}
\begin{aligned}
&\textup{Tr}((\phi(\bm{X})^T\phi(\bm{X}))^{1/2})=\textup{Tr}((\bm{V}\bm{S}^2\bm{V}^T)^{1/2})\\
=&\textup{Tr}(\bm{V}(\bm{S}^2)^{1/2}\bm{V}^T)=\textup{Tr}(\bm{S})=\Vert \phi(\bm{X})\Vert_\ast.
\end{aligned}
\end{equation}
Substituting (\ref{Eq.SpNormT}) into problem (\ref{Eq.RKPCA_0}), we get
\begin{equation}\label{Eq.RKPCA_1}
\min\limits_{\bm{X},\bm{E}}\textup{Tr}((\phi(\bm{X})^T\phi(\bm{X}))^{1/2})+\lambda\Vert \bm{E}\Vert_1,\ s.t.\ \bm{X}+\bm{E}=\bm{M}.
\end{equation}
In the $n\times n$ Gram matrix $\phi(\bm{X})^T\phi(\bm{X})$, each element $\phi(\bm{x}_i)^T\phi(\bm{x}_j)=\langle\phi(\bm{x}_i),\phi(\bm{x}_j)\rangle$ can be directly obtained by kernel representations \cite{KPCA1998}, i.e.,
\begin{equation}\label{Eq.KerMap}
\langle\phi(\bm{x}_i),\phi(\bm{x}_j)\rangle=k(\bm{x}_i,\bm{x}_j),
\end{equation}
without carrying out $\phi(\cdot)$ explicitly. In (\ref{Eq.KerMap}), $k(\cdot,\cdot)$ is a kernel function satisfying Mercer's theorem \cite{KPCA1998}. With a certain kernel function $k(\cdot,\cdot)$, the nonlinear mapping $\phi(\cdot)$ and the feature space $\mathcal{F}$ will be implicitly determined. The most widely-used kernel function is the radial basis function (RBF) kernel
\begin{equation}\label{Eq.K_RBF}
k(\bm{x},\bm{y})=\exp\left(-\Vert \bm{x}-\bm{y}\Vert^2/(2\sigma^2)\right),
\end{equation}
where $\sigma$ is a free parameter controlling the smoothness degree of the kernel. We give the following lemma:
\begin{lemma}
The features given by RBF kernel are Taylor features of the input data and the feature dimensionality is infinite \cite{DBLP:journals/corr/abs-1109-4603}.
\end{lemma}
\begin{proof}
RBF kernel can be written as
\begin{equation}\label{Eq.K_RBF_1}
\begin{aligned}
k(\bm{x},\bm{y})&=\exp\left(-\dfrac{1}{2\sigma^2}\Vert \bm{x}-\bm{y}\Vert^2\right)\\
&=\exp\left(-\dfrac{1}{2\sigma^2}\left(\Vert \bm{x}\Vert^2+\Vert \bm{y}\Vert^2-2\langle \bm{x},\bm{y}\rangle\right)\right)\\
&=C\exp\left(\dfrac{1}{\sigma^2}\langle \bm{x},\bm{y}\rangle\right)\\
&=C\sum_{n=0}^\infty\dfrac{\langle \bm{x},\bm{y}\rangle^n}{\sigma^{2n}n!}
\end{aligned}
\end{equation}
where $C=\exp\left(-\dfrac{1}{2\sigma^2}\left(\Vert \bm{x}\Vert^2+\Vert \bm{y}\Vert^2\right)\right)$. Because $\langle \bm{x},\bm{y}\rangle^n$ is $n$-th order polynomial kernel, (\ref{Eq.K_RBF_1}) means that RBF kernel is a weighted sum of polynomial kernel of different orders from $0$ to $\infty$. Then the features given by RBF kernel are
\begin{equation}
\begin{aligned}
\phi(\bm{x})=&[c_0,c_1x_1,\cdots,c_dx_d,\cdots,c_{ij}x_ix_j,\cdots,\\
&c_{ijk}x_ix_jx_k,\cdots,c_{i\cdots k}x_i\cdots x_k,\cdots]^T,
\end{aligned}
\end{equation}
which is the Taylor features of $\bm{x}$. The dimensionality of $\phi(\bm{x})$ is infinity.
\end{proof}
As $\bm{x}=f(\bm{z})$ can be approximated by Taylor expansion, the features given by RBF kernel are also Taylor features of the latent variable $\bm{z}$, e.g.
\begin{equation}\label{poly_z}
\begin{aligned}
\phi(\bm{x})=&[b_0,b_1z_1,\cdots,b_rz_r,\cdots,b_{ij}z_iz_j,\cdots,\\
&b_{ijk}z_iz_jz_k,\cdots,b_{i\cdots k}z_i\cdots z_k,\cdots]^T.
\end{aligned}
\end{equation}
That is why we use Taylor series to estimate the rank of $\phi(\bm{X})$ in Theorem \ref{asp_kernel}.  

With the kernel matrix $\bm{K}=[k(\bm{x}_i,\bm{x}_j)]_{n\times n}=[\phi(\bm{x}_i)^T\phi(\bm{x}_j)]_{n\times n}$, problem (\ref{Eq.RKPCA_1}) can be rewritten as
\begin{equation}\label{Eq.RKPCA_2}
\min\limits_{\bm{X},\bm{E}}\textup{Tr}(\bm{K}^{1/2})+\lambda\Vert \bm{E}\Vert_1,\ s.t.\ \bm{X}+\bm{E}=\bm{M}.
\end{equation}
According to Lemma \ref{lem_rec}, since $\textup{Tr}(\bm{K}^{1/2})$ and $\Vert \bm{E}\Vert_1$ are the relaxations of $\textup{ldim}(\bm{X})$ and $\Vert \bm{E}\Vert_0$ in (\ref{Eq.minldim}), $\bm{X}$ and $\bm{E}$ can be recovered through solving (\ref{Eq.RKPCA_2}) with some appropriate $\lambda$. We call the proposed method robust kernel PCA (RKPCA), which is able to handle nonlinear data and sparse noises simultaneously. Moreover, since the nonlinear data often form high-rank matrices, RKPCA is able to recover $\bm{X}$ and $\bm{E}$ even if $\bm{X}$ is of high-rank or full-rank. On the contrary, RPCA cannot effectively recover high-rank matrices and full-rank matrices.

The parameter $\lambda$ is crucial to RKPCA. In (\ref{Eq.RKPCA_2}), $\textup{Tr}(\bm{K}^{1/2})$ and $\Vert \bm{E}\Vert_1$ may have different orders of magnitude. When RBF kernel is used, we have
\begin{equation}
\sqrt{n}<\textup{Tr}(\bm{K}^{1/2})<n.
\end{equation}
To balance the two terms in (\ref{Eq.RKPCA_2}) or (\ref{Eq.RKPCA_0.5}), $\lambda$ should be determined as
\begin{equation}
\lambda=n\lambda_0/\Vert \bm{M}\Vert_1,
\end{equation} 
where the value of $\lambda_0$ can be chosen around 0.5.

It is worth noting that in RKPCA the nuclear norm $\Vert \phi(\bm{X})\Vert_\ast$ is a special case (when $p=1$) of the Schatten $p$-norm $\Vert \phi(\bm{X})\Vert_{Sp}=\left(\sum_{i=1}^{min(\kappa,n)}(\sigma_i)^p\right)^{1/p}$ ($0<p\leq 1$), which is a nonconvex relaxation of matrix rank. We have $\Vert \phi(\bm{X})\Vert_{Sp}^p=\textup{Tr}(\bm{K}^{p/2})$ and get the following generalized form of RKPCA
\begin{equation}\label{Eq.RKPCA_sp}
\min\limits_{\bm{X},\bm{E}}\textup{Tr}(\bm{K}^{p/2})+\lambda\Vert \bm{E}\Vert_1,\ s.t.\ \bm{X}+\bm{E}=\bm{M}.
\end{equation}

A geometrical interpretation is as follow. In terms of the RBF kernel, if the distance between two data points $\bm{x}_i$ and $\bm{x}_j$ is large, $k(\bm{x}_i,\bm{x}_j)$ will be very small. It means that the element $\phi^T(\bm{x}_i)\phi(\bm{x}_j)=k(\bm{x}_i,\bm{x}_j)$ makes little contribution to the objective function of RKPCA. Therefore RKPCA exploits more local information than global information. That  is why RKPCA can handle nonlinearity. When the RBF kernel is replaced with linear kernel $k(\bm{x}_i,\bm{x}_j)=\phi^T(\bm{x}_i)\phi(\bm{x}_j)=\bm{x}_i^T\bm{x}_j$ and $p=1$, RKPCA is identical to RPCA. Because linear kernel cannot to recognize local information, RPCA is unable to handle nonlinearity.

\subsection{Theoretical guarantee of effective recovery}
In the following content, we will provide theoretical guarantee in terms of noise density for RKPCA. To determine $\bm{X}$ from $\bm{M}$, the number of clean entries should be larger than the number of degrees of freedom of $\bm{X}$, which is the minimum number of parameters to fix $\bm{X}$. First, because the latent dimension of $\bm{X}$ is $r$, to determine one column $\bm{x}\in\mathbb{R}^d$, we need to observe at least $r$ clean entries (denoted by $\bm{x}^o$) to form a set of bases. Therefore, for all columns of $\bm{X}$, we need $nr$ parameters. Second, in each column of $\bm{X}$, the remaining $d-r$ entries (denoted by $\bm{x}^{\bar{o}}$) can be reconstructed as 
\begin{equation}
\bm{x}^{\bar{o}}=g(\bm{x}^o),
\end{equation}
where $g:\mathbb{R}^r\rightarrow \mathbb{R}^{d-r}$ is an unknown nonlinear mapping. According to (\ref{poly_z}), we should approximate $g$ by polynomials. The Taylor series of $g$ at $\bm{x}^o$ is given by
\begin{equation}\label{Eq.taylor_1}
g(\bm{x}^o)=\sum_{\vert \alpha \vert\geq 0}\dfrac{(\bm{x}^o-\bm{v})^{\alpha}}{\alpha !}(\partial^{\alpha} g)(\bm{v}).
\end{equation}
Then the $t$-th order Taylor approximation of $g(\bm{x}^o)$ is 
\begin{equation}\label{Eq.taylor_2}
g(\bm{x}^o)\approx \bm{C}_0+\bm{C}_1\bm{\chi}_1+\cdots+\bm{C}_t\bm{\chi}_t,
\end{equation}
where $\bm{C}_0\in\mathbb{R}^{(d-r)\times 1}$. For $k=1,\cdots,t$, $\bm{C}_k\in\mathbb{R}^{(d-r)\times d_k}$, $\bm{\chi}_k\in\mathbb{R}^{d_k\times 1}$ consists of the $k$-combinations (product) of the elements in $\bm{x}^o-\bm{v}$, and $d_k=\binom{r+k-1}{k}$. Hence, to approximate $g$, we need to determine $\lbrace\bm{C}_0,\bm{C}_1,\cdots,\bm{C}_t\rbrace$, which has $(d-r)\times (1+\sum_{k=1}^t\binom{r+k-1}{k})=(d-r)\times \binom{r+t}{t}$ parameters. Third, we need at least one more entry of each column of $\bm{X}$ to validate the positions of the clean entries and noisy entries and hence require $n$ parameters for all columns of $\bm{X}$. Summarizing the numbers of the parameters required in the three steps, we get the number of degrees of freedom of $\bm{X}$ as $(r+1)n+(d-r)\times \binom{r+t}{t}$. Therefore, the noise density should meet the following condition
\begin{equation}\label{Eq.rho_dof}
\rho<1-\dfrac{r+1}{d}-\dfrac{d-r}{dn}\times \binom{r+t}{t}.
\end{equation}
Combining (\ref{Eq.rho_mu}) and (\ref{Eq.rho_dof}), the following lemma holds:
\begin{lemma}\label{lem.rho_bound}
Suppose $\bm{X}\in\mathbb{R}^{d\times n}$ is given by Assumption \ref{assump_0} and $\bm{M}=\bm{X}+\bm{E}$, it is possible to approximately recover $\bm{X}$ from $\bm{M}$ if the noise density of $\bm{E}$ meets $\rho<\min\lbrace\mu_0 d,1-(r+1)/d-(d-r)\binom{r+t}{t}/(dn)\rbrace$.
\end{lemma}

In Lemma \ref{lem.rho_bound}, the parameter $t$ determines the recovery error, which is on the order of $\partial^{t+1}g/(t+1)!$. When $g$ can be perfectly approximated by its $t$-th order Taylor expansion, the recovery error is zero. However Lemma \ref{lem.rho_bound} is a necessary condition for recovering $\bm{X}$. For instance, suppose 
\begin{equation}
\left\lbrace
\begin{aligned}
&x_1=f_1(z_1),&&x_2=f_2(z_1),\\
&x_3=f_3(z_1,z_2),&&x_4=f_4(z_1,z_2),
\end{aligned}
\right.
\end{equation}
when $\lbrace x_3,x_4\rbrace$ are corrupted and $\lbrace x_1,x_2\rbrace$ are clean, it is impossible to recover $\lbrace x_3,x_4\rbrace$ because the information of $z_2$ is lost. Therefore the noisy density $\rho$ should be lower such that the information of $z_1$ and $z_2$ can be preserved. Let $\bm{J}$ be the Jacobian matrix of $f$, i.e.,
\begin{equation}
\bm{J}=\dfrac{\partial \bm{x}}{\partial \bm{z}}=\left[
\begin{matrix}
\tfrac{\partial{x_1}}{\partial{z_1}} & \tfrac{\partial{x_1}}{\partial{z_2}} & \cdots &\tfrac{\partial{x_1}}{\partial{z_r}}\\
\tfrac{\partial{x_2}}{\partial{z_1}} & \tfrac{\partial{x_2}}{\partial{z_2}} & \cdots &\tfrac{\partial{x_2}}{\partial{z_r}}\\
\vdots & \vdots & \ddots & \vdots\\
\tfrac{\partial{x_d}}{\partial{z_1}} & \tfrac{\partial{x_d}}{\partial{z_2}} & \cdots &\tfrac{\partial{x_d}}{\partial{z_r}}\\
\end{matrix}
\right],
\end{equation}
and $\bm{J}_{ij}=\tfrac{\partial{x_i}}{\partial{z_j}}$. Define 
\begin{equation}
\delta=\textup{Pr}(\bm{J}_{ij}\neq 0)\ \textup{or}\ \delta=\textup{Pr}(\vert\bm{J}_{ij}\vert>\epsilon),
\end{equation}
where $\epsilon$ is a small positive constant. $\delta$ measures the non-sparsity of $\bm{J}$. We give the following lemma:
\begin{lemma}\label{lem_incoherence}
Suppose that $\mu r$ elements of each column of $\bm{M}$ are clean, where $1\leq\mu\leq d/r$. Then the information of all elements of $\bm{Z}$ can be preserved with high probability, provided that $\mu$ is large enough. 
\end{lemma}
\begin{proof}
Let $\lbrace x_i\rbrace_{i\in \mathbb{S}}$ be a subset of $\bm{x}$, where $\mathbb{S}$ consists of $\mu r$ distinct elements of $\lbrace 1,2,\cdots, d\rbrace$. We obtain
\begin{equation}
\textup{Pr}(\tfrac{\partial x_i}{\partial z_1}=0, \forall i\in\mathbb{S})=(1-\delta)^{\mu r}.
\end{equation}
It follows that 
\begin{equation}
\textup{Pr}(\tfrac{\partial x_i}{\partial z_1}=0\cup\cdots\cup\tfrac{\partial x_i}{\partial z_r}=0, \forall i\in\mathbb{S})=r(1-\delta)^{\mu r}.
\end{equation}
Then the information of all elements of $\bm{Z}$ can be preserved with probability 
\begin{equation}
\rho_0=1-nr(1-\delta)^{\mu r}.
\end{equation}
To ensure a high $\rho_0$, $\mu$ should be large enough and $\mu=\tfrac{1}{r}\log_{1-\delta}\tfrac{1-\rho_0}{nr}$.  For example, let $r=5$, $n=100$, and $\delta=0.6$, we have $\rho_0=0.9467$ if $\mu=2$ and $\rho_0=0.9995$ if $\mu=3$.
\end{proof}
It is worth noting that Lemma \ref{lem_incoherence} is similar to the incoherence condition in RPCA \cite{RPCA}. In RPCA, the incoherence property is defined on the singular vectors of $\bm{X}$ and measures the non-spiky property of the singular vectors. In our paper, $\bm{X}$ is given by Assumption \ref{assump_0} and cannot be formulated by singular value decomposition. Hence the incoherence property of RPCA is inapplicable to our model. In fact, Lemma \ref{lem_incoherence} provides a condition for that the Jacobian matrix defined by the clean elements of $\bm{x}$ is of full-rank, which ensures that $f$ is invertible at least locally \cite{clarke1976inverse} and hence the corrupted elements of $\bm{x}$ can be recovered.

In accordance with RPCA \cite{RPCA}, we also assume that the locations of nonzero entries of $\bm{E}$ are independently determined by Bernoulli distribution, i.e., for all $(i,j)$, $\textup{Pr}(\bm{E}_{ij}\neq 0)=\rho$ and $\textup{Pr}(\bm{E}_{ij}=0)=1-\rho$. The following theorem shows that $\bm{X}$ can be recovered with high probability.
\begin{theorem}\label{theorem_sufficient}
Suppose $\bm{X}\in\mathbb{R}^{d\times n}$ is given by Assumption \ref{assump_0} and $\bm{M}=\bm{X}+\bm{E}$, with probability at least $(1-nr(1-\delta)^{\mu r})(1-n^{1-c})$, $\bm{X}$ can be recovered with error $O(\partial^{t+1}g/(t+1)!)$ from $\bm{M}$ if the noise density of $\bm{E}$ satisfies
\begin{equation}
\rho<1-\dfrac{\max\lbrace c\mu_1rn\log n,(r+1)n+(d-r)\binom{r+t}{t}\rbrace}{dn},
\end{equation}
where $c$ is a numerical constant and $\mu_1=\max\lbrace 1+\tfrac{1}{r},\mu\rbrace$.
\end{theorem}
\begin{proof}
According to Lemma \ref{lem_incoherence} and the number of degrees of freedom of $\bm{X}$ with $O(\partial^{t+1}g/(t+1)!)$ residuals, the number of clean entries of each column of $\bm{X}$ should be at least $\max\lbrace r+1,\mu r\rbrace\triangleq \mu_1r$. With the Bernoulli model of clean entry locations, to ensure $\mu_1r$ clean entries for every column of $\bm{X}$ with probability 
\begin{equation}
\rho_1=1-n^{1-c}, 
\end{equation}
the number of clean entries of $\bm{X}$ should be at least $c\mu_1rn\log n$ (the coupon collector's problem\footnote{In the coupon collector's problem \cite{CCP1997}, through $cn\log n$ trials, one can collect $n$ different coupons with probability at least $1-n^{1-c}$.}\cite{CCP1997,CandesRecht2009}), where $c>1$ is a numerical constant. Combining Lemma \ref{lem.rho_bound}, when the number of clean entries of $\bm{X}$ is $\max\lbrace (1-\mu_0d)dn, (r+1)n+(d-r)\binom{r+t}{t}, c\mu_1rn\log n\rbrace$, $\bm{X}$ can be approximately recovered with probability at least $\rho_0\rho_1$. Since $\mu_0>1/d$ often holds easily, the condition becomes $\max\lbrace (r+1)n+(d-r)\binom{r+t}{t}, c\mu_1rn\log n\rbrace$. It indicates that when the noise density of $\bm{E}$ meets $\rho<1-\max\lbrace c\mu_1rn\log n,(r+1)n+(d-r)\binom{r+t}{t}\rbrace/(dn)$, $\bm{X}$ can be approximately recovered with probability at least $\rho_0\rho_1$. The probability will be high if $c$ and $\mu$ are large enough. In addition, larger $t$ leads to lower recover error, provided that the condition of $\rho$ holds.
\end{proof}

Theorem \ref{theorem_sufficient} indicates that RKPCA can recover $\bm{X}$ if $\rho<1-C_1\textup{ldim}(\bm{X})/d$, where $C_1$ is a numerical constant, provided that $n$ is large enough. In Theorem \ref{theorem_sufficient}, when $\textup{rank}(\bm{X})=\textup{ldim}(\bm{X})=r$, it indicates that RPCA can recover $\bm{X}$ if $\rho<1-C_2\textup{rank}(\bm{X})/d$, where $C_2$ is a numerical constant, provided that $n$ is large enough. We also have $C_1>C_2$ but $C_1\approx C_2$ when $n$ is large enough. When $\bm{X}$ has nonlinear structures, $\textup{rank}(\bm{X})\gg \textup{ldim}(\bm{X})$ and then $1-C_1\textup{ldim}(\bm{X})/d>1-C_2\textup{rank}(\bm{X})/d$. It means that RKPCA can outperform RPCA in recovering high-rank and full-rank matrices.

More generally, the columns of $\bm{X}\in\mathbb{R}^{d\times n}$ could be drawn from multiple nonlinear latent variable models, e.g.
\begin{equation}\label{Eq.X_MS}
\bm{X}=[f_1(\bm{Z}_1),f_2(\bm{Z}_2),\cdots,f_k(\bm{Z}_k)],
\end{equation}
where, for $j=1,2,\cdots,k$, $f_j:\mathbb{R}^{r_j}\rightarrow\mathbb{R}^{d}$ is a nonlinear mapping performed on the columns of $\bm{Z}_j$ separately, $\bm{Z}_j\in\mathbb{R}^{r_j\times n_j}$, and $\sum_{j=1}^kn_j=n$. We can also say that $\bm{X}$ is drawn from a union ($k$) of subspaces nonlinearly. $\bm{X}$ is often of full-rank. Let $\bm{Z}=[\bm{Z_1},\bm{Z_2},\cdots,\bm{Z}_k]$ and denote $\bm{s}\in\mathbb{R}^{1\times n}$ as the labels of the columns of $\bm{X}$, then we reformulate (\ref{Eq.X_MS}) as
\begin{equation}
\bm{X}=f\left(
\begin{matrix}
\bm{Z}\\ \bm{s}
\end{matrix} 
\right)\triangleq f(\bar{\bm{Z}}).
\end{equation}
It means that such a matrix $\bm{X}$ is still in according with the assumption of RKPCA. The difference is that the latent dimension increased and the nonlinear mapping $f$ is more complex. According to (\ref{Eq.rank_phiX}), we have 
\begin{equation}\label{Eq.rank_phiX_k}
\textup{rank}(\phi(\bm{X}))\approx k\binom{r+t}{t},
\end{equation}
where we have assumed $r_1=\cdots=r_k=r$ for convenience and $n_1,\cdots,n_k$ are sufficiently large. Therefore, $\bm{X}$ can be recovered by RKPCA when $\rho<1- O\left(\textup{ldim}(\bm{X})/d\right)$ provided that $n$ is large enough.

\subsection{Optimization via ADMM plus backtracking line search (ADMM+BTLS)}
In RKPCA formulated by (\ref{Eq.RKPCA_2}), the two blocks of unknown variables $\bm{X}$ and $\bm{E}$ are separable. The second term $\lambda\Vert \bm{E}\Vert_1$ and the constraint are convex. The first term $\textup{Tr}(\bm{K}^{1/2})$ could be nonconvex if a nonlinear kernel such as RBF kernel is used. We propose to solve (\ref{Eq.RKPCA_2}) via alternating direction method of multipliers (ADMM) \cite{ADMM,doi:10.1137/140990309,Deng2016}. The augmented Lagrange function of (\ref{Eq.RKPCA_2}) is given by
\begin{equation}
\begin{aligned}
&\mathcal{L}(\bm{X},\bm{E},\bm{Q})=\textup{Tr}(\bm{K}^{1/2})+\lambda\Vert \bm{E}\Vert_1\\
&+\langle \bm{X}+\bm{E}-\bm{M}, \bm{Q}\rangle+\dfrac{\mu}{2}\Vert \bm{X}+\bm{E}-\bm{M}\Vert_F^2,
\end{aligned}
\end{equation}
where $\bm{Q}\in\mathbb{R}^{d\times n}$ is the matrix of Lagrange multipliers and $\mu>0$ is a penalty parameter. $\mathcal{L}(\bm{X},\bm{E},\bm{Q})$ can be minimized over $\bm{X}$ and $\bm{E}$ one after another. 

First, fix $\bm{E}$ and solve
\begin{equation}\label{Eq.solve_X}
\min\limits_{\bm{X}} \textup{Tr}(\bm{K}^{1/2})+\dfrac{\mu}{2}\Vert \bm{X}+\bm{E}-\bm{M}+\bm{Q}/\mu\Vert_F^2.
\end{equation}
Because of the kernel matrix $\bm{K}$, problem (\ref{Eq.solve_X}) has no closed-form solution. Denoting 
\begin{equation}\label{Eq.J}
\mathcal{J}=\textup{Tr}(\bm{K}^{1/2})+\dfrac{\mu}{2}\Vert \bm{X}+\bm{E}-\bm{M}+\bm{Q}/\mu\Vert_F^2,
\end{equation}
we propose to solve (\ref{Eq.solve_X}) with gradient descent, i.e.,
\begin{equation}\label{Eq.solve_X_updata}
\bm{X}\leftarrow \bm{X}-\eta \dfrac{\partial \mathcal{J}}{\partial \bm{X}}.
\end{equation}
where $\eta$ is the step size and will be discussed later. Denoting $\mathcal{J}_1=\textup{Tr}(\bm{K}^{1/2})$ and $\mathcal{J}_2=\dfrac{\mu}{2}\Vert \bm{X}+\bm{E}-\bm{M}+\bm{Q}/\mu\Vert_F^2$, we have $\partial \mathcal{J}_2/\partial \bm{X}=\mu(\bm{X}+\bm{E}-\bm{M}+\bm{Q}/\mu)$. $\partial \mathcal{J}_1/\partial \bm{X}$ can be computed by the chain rule
\begin{equation}\label{Eq.dLdX}
\dfrac{\partial{\mathcal{J}_1}}{\partial{\bm{X}}}=\sum_{i=1}^n\sum_{j=1}^n\dfrac{\partial{\mathcal{J}_1}}{\partial{K_{ij}}}\dfrac{\partial{K_{ij}}}{\partial{\bm{X}}},
\end{equation}
where
\begin{equation}
\dfrac{\partial \mathcal{J}_1}{\partial \bm{K}}=\dfrac{1}{2}\bm{K}^{-\tfrac{1}{2}}.
\end{equation}

In the second step of minimizing $\mathcal{L}(\bm{X},\bm{E},\bm{Q})$, we fix $\bm{X}$ and solve
\begin{equation}\label{Eq.solve_E}
\min\limits_{\bm{E}} \lambda\Vert \bm{E}\Vert_1+\dfrac{\mu}{2}\Vert \bm{X}+\bm{E}-\bm{M}+\bm{Q}/\mu\Vert_F^2.
\end{equation}
The closed-form solution of (\ref{Eq.solve_E}) is given as
\begin{equation}\label{Eq.solve_E_result}
\bm{E}=\Theta_{\lambda/\mu}(\bm{M}-\bm{X}-\bm{Q}/\mu),
\end{equation}
where $\Theta(\cdot)$ is the element-wise soft thresholding operator \cite{ProximalA} defined as
\begin{equation}
\Theta_\tau(u)=:\dfrac{u}{\vert u\vert}max\lbrace\vert u\vert-\tau,0\rbrace.
\end{equation}
Finally, the matrix of Lagrange multipliers is updated by
\begin{equation}
\bm{Q}\leftarrow \bm{Q}+\mu(\bm{X}+\bm{E}-\bm{M}).
\end{equation}
The optimization frame for RKPCA is summarized in Algorithm \ref{alg.RKPCA}. The convergence condition is
\begin{equation}
\dfrac{\mathcal{L}^{(t-1)}(\bm{X},\bm{E},\bm{Q})-\mathcal{L}^{(t)}(\bm{X},\bm{E},\bm{Q})}{\mathcal{L}^{(t-1)}(\bm{X},\bm{E},\bm{Q})}<\varepsilon
\end{equation}
where $\varepsilon$ is a small number such as $10^{-6}$.

\renewcommand{\algorithmicrequire}{\textbf{Input:}}
\renewcommand{\algorithmicensure}{\textbf{Output:}}
\begin{algorithm}[h]
\caption{Solve RKPCA with ADMM}
\label{alg.RKPCA}
\begin{algorithmic}[1] 
\Require
$\bm{M}$, $k(\cdot,\cdot)$, $\lambda$, $\mu$, $t_{max}$.
\State \textbf{initialize:} $\bm{X}^{(0)}=\bm{M}$, $\bm{E}^{(0)}=0$, $t=0$.
\Repeat
\State $\bm{X}^{(t+1)}= \bm{X}^{(t)}-\eta (\dfrac{\partial \mathcal{J}_1}{\partial \bm{X}^{(t)}}+\dfrac{\partial \mathcal{J}_2}{\partial \bm{X}^{(t)}})$
\State $\bm{E}^{(t+1)}=\Theta_{\lambda/\mu}(\bm{M}-\bm{X}^{(t+1)}-\bm{Q}/\mu)$
\State $\bm{Q}\leftarrow \bm{Q}+\mu(\bm{X}^{(t+1)}+\bm{E}^{(t+1)}-\bm{M})$
\State $t=t+1$
\Until{converged or $t=t_{max}$}
\Ensure $\bm{X}=\bm{X}^{(t)}$, $\bm{E}=\bm{E}^{(t)}$
\end{algorithmic}
\end{algorithm}

In Algorithm \ref{alg.RKPCA}, a large $\eta$ will increase the convergence speed but may make (\ref{Eq.solve_X}) diverge. It is known that the subproblem (\ref{Eq.solve_X}) will converge when $\eta<1/L_\mathcal{J}$, where $L_\mathcal{J}$ is the Lipschitz constant of the gradient of $\mathcal{J}$ in (\ref{Eq.J}) \cite{doi:10.1137/140990309,Deng2016}. The reason is that the gradient descent is the solution of the $L_\mathcal{J}$ quadratic approximation, which is strongly convex \cite{ADMM}. In kernel methods such as SVM, KPCA and Gaussian processes, RBF kernel usually outperforms polynomial kernel. Similarly, we find that in RKPCA, RBF kernel is more effective than polynomial kernel. Hence, in this study, we mainly focus on RBF kernel. As suggested in a lot of work of kernel methods, the parameter $\sigma$ can be chosen around the average of pair-wise distance of all data points \cite{chapelle2005semi}, i.e., 
\begin{equation}\label{Eq.kerpar}
\sigma=\dfrac{\beta}{n^2}\sum_{i=1}^n\sum_{j=1}^n\Vert \bm{x}_i-\bm{x}_j\Vert,
\end{equation}
where $\beta$ can be chosen from $\lbrace 0.5,1,1.5,2\rbrace$.
With RBF kernel, the gradient $\partial \mathcal{J}_1/\partial \bm{X}$ can be computed as
\begin{equation}\label{Eq.J1}
\dfrac{\partial \mathcal{J}_1}{\partial \bm{X}}=\dfrac{2}{\sigma^2}(\bm{X}\bm{H}-\bm{X}\odot (\bm{B}\bm{H})),
\end{equation}
where $\bm{H}=\dfrac{\partial \mathcal{J}_1}{\partial \bm{K}}\odot \bm{K}=\dfrac{1}{2}\bm{K}^{-1/2}\odot \bm{K}$, $\bm{B}$ is a $d\times n$ matrix consisting of $1$s, and $\odot$ denotes the Hadamard product. As can be seen, it is very difficult to obtain an available Lipschitz constant of the gradient because of the presence of kernel matrix. However, we can use inexact line search to find a good step size $\eta$ that meets the sufficient decrease condition (Armijo-Goldstein inequality, Chapter 3 in \cite{nocedal2006nonlinear})
\begin{equation}
\mathcal{J}(\bm{X}-\eta \dfrac{\partial \mathcal{J}}{\partial \bm{X}})<\mathcal{J}(\bm{X})-\gamma\eta\Vert \dfrac{\partial \mathcal{J}}{\partial \bm{X}}\Vert_F^2,
\end{equation}
where $0<\gamma<1$. The procedures of backtracking line search for $\eta$ are shown in Algorithm \ref{alg.LineSearch}. With the obtained $\eta$, the step 3 in Algorithm \ref{alg.RKPCA} will be non-expansive and the augmented Lagrange function will decrease sufficiently. 

\begin{algorithm}[h]
\caption{Backtracking line search for $\eta$}
\label{alg.LineSearch}
\begin{algorithmic}[1] 
\Require
$\eta=\eta_0$, $c$ (e.g., $0.5$), $\gamma$ (e.g., $0.1$).
\While {$\mathcal{J}(\bm{X}-\eta \dfrac{\partial \mathcal{J}}{\partial \bm{X}})>\mathcal{J}(\bm{X})-\gamma\eta\Vert \dfrac{\partial \mathcal{J}}{\partial \bm{X}}\Vert_F^2$}
\State $\eta=c\eta$
\EndWhile
\Ensure $\eta$
\end{algorithmic}
\end{algorithm}

In Algorithm \ref{alg.LineSearch}, we need to initialize $\eta$ beforehand. A large $\eta_0$ may provide a broad search region to obtain a better solution but will need more iterations. We can compute a coarse estimation of $L_\mathcal{J}$ to initialize $\eta$ as
\begin{equation}
\eta_0=10/L_\mathcal{J},
\end{equation}
where
\begin{equation}\label{Eq.Lf}
L_\mathcal{J}=\Vert \dfrac{2}{\sigma^2}(\bm{H}-\varsigma\bm{I})+\mu \bm{I}\Vert_2.
\end{equation}
In (\ref{Eq.Lf}), $\bm{I}$ is an $n\times n$ identity matrix and $\Vert \cdot\Vert_2$ denotes the spectral norm of matrix. The estimation of $L_\mathcal{J}$ is from $\partial \mathcal{J}/\partial \bm{X}$, in which we have regarded $\bm{H}$ as a constant matrix not involved with $\bm{X}$ and replaced $\bm{X}\odot (\bm{B}\bm{H})$ with $\varsigma \bm{X}$. The reason is that $\bm{H}$ is related with $\bm{K}$ and the scaling factor $\sigma^2$ has made the impact of $\bm{X}$ quite small. In addition, the elements of each row of $\bm{B}\bm{H}$ are the same and hence we replace $\odot \bm{B}\bm{H}$ by a number $\varsigma$, which is the average of all elements of $\bm{B}\bm{H}$. The derivation of (\ref{Eq.Lf}) is detailed as follows. 
\begin{equation}\label{Eq.estL}
\begin{aligned}
&\Vert \dfrac{\partial \mathcal{J}}{\partial \bm{X}_a}-\dfrac{\partial \mathcal{J}}{\partial \bm{X}_b}\Vert_F=\Vert \dfrac{2}{\sigma^2}(\bm{X}_a\bm{H}_a-\bm{X}_a\odot (\bm{B}\bm{H}_a)\\
&-\bm{X}_b\bm{H}_b+\bm{X}_b\odot (\bm{B}\bm{H}_b))+\mu(\bm{X}_a-\bm{X}_b)\Vert_F\\
\approx & \Vert \dfrac{2}{\sigma^2}(\bm{X}_a\bm{H}-\bm{X}_a\odot (\bm{B}\bm{H})\\
&-\bm{X}_b\bm{H}+\bm{X}_b\odot (\bm{B}\bm{H}))+\mu(\bm{X}_a-\bm{X}_b)\Vert_F\\
\approx &\Vert \dfrac{2}{\sigma^2}(\bm{X}_a\bm{H}-\varsigma \bm{X}_a-\bm{X}_b\bm{H}+\varsigma \bm{X}_b)+\mu(\bm{X}_a-\bm{X}_b)\Vert_F\\
= &\Vert (\bm{X}_a-\bm{X}_b)(\dfrac{2}{\sigma^2}(\bm{H}-\varsigma\bm{I}) +\mu \bm{I})\Vert_F\\
\leq &\Vert \bm{X}_a-\bm{X}_b\Vert_F \Vert(\dfrac{2}{\sigma^2}(\bm{H}-\varsigma\bm{I}) +\mu \bm{I}\Vert_2\\
\end{aligned}
\end{equation}

Therefore, $\alpha L_\mathcal{J} (0<\alpha<1)$ could be an estimation of the Lipschitz constant of $\dfrac{\partial \mathcal{J}}{\partial \bm{X}}$ if $\alpha$ is small enough. We found that when $\eta_0=10/L_\mathcal{J}$, Algorithm \ref{alg.LineSearch} can find $\eta$ in at most 5 iterations.

The optimization of RKPCA is in the nonconvex framework of ADMM studied in \cite{doi:10.1137/140990309} (problem (3.2) and Algorithm 4 of the paper). The corresponding case is
\begin{equation}\label{Eq.ADMM_ref}
\min\limits_{\bm{u},\bm{v}} f(\bm{u})+g(\bm{v}),\ s.t. \bm{u}+\bm{A}\bm{v}=\bm{c},
\end{equation}
where $f(\cdot)$ is nonconvex and $g(\cdot)$ is convex but nonsmooth. In our problem, $f(\cdot)$ is $\textup{Tr}(\bm{K}^{1/2})$, $g(\cdot)$ is $\lambda\Vert\bm{E}\Vert_1$, $\bm{A}$ is an identity matrix, and $\bm{c}$ is $\bm{M}$. In \cite{doi:10.1137/140990309}, it was shown that ADMM for (\ref{Eq.ADMM_ref}) is able to converge to the set of stationary solutions, provided that the penalty parameter in the augmented Lagrange is large enough. More detailedly, the following conditions are required in the proximal version of solving (\ref{Eq.ADMM_ref}): (a) $f(\cdot)$ is smooth with $L$-Lipschitz continuous gradient and $\bm{A}$ is of full column rank; (b) the two subproblems do not diverge and the corresponding objective functions descrease sufficiently; (c) $\mu\geq L$. In our optimization, the Lipschitz constant of $f(\bm{u})$'s gradient is estimated as $L=\Vert 2(\bm{H}-\varsigma\bm{I})/{\sigma^2}\Vert_2$ according to (\ref{Eq.estL}) and $\bm{A}$ is an identity matrix. Therefore, condition (a) holds. The two subproblems are handled by backtracking line search and soft thresholding respectively, which meets condition (b). In addition, condition (c) holds when $\sigma$ and $\mu$ are large enough. Therefore the convergence of Algorithm \ref{alg.RKPCA} can be proved accordingly through the similar approach of \cite{doi:10.1137/140990309}, which will not be detailed in this paper. Similar to the optimization of RPCA and other methods solved by ADMM, the Lagrange penalty $\mu$ in Algorithm \ref{alg.RKPCA} is quite easy to determine. We can just set $\mu=10\lambda$. In fact, ADMM is not sensitive to the Lagrange penalty.

\subsection{Optimization via proximal linearized minimization with adaptive step size (PLM+AdSS)}
In RKPCA, problem (\ref{Eq.RKPCA_2}) can be rewritten as
\begin{equation}\label{Eq.RKPCA_3}
\min\limits_{\bm{E}}\textup{Tr}(\bm{K}^{1/2})+\lambda\Vert \bm{E}\Vert_1,
\end{equation}
where $\bm{K}$ is the kernel matrix carried out on $\bm{M}-\bm{E}$. We denote the objective function of (\ref{Eq.RKPCA_3}) by $\mathcal{J}$. RKPCA in the form of (\ref{Eq.RKPCA_3}) has no constraint and has only one block of variables. Hence, we propose to solve (\ref{Eq.RKPCA_3}) by proximal linearied minimization with adaptive step size (PLM+AdSS), which is shown in Algorithm \ref{alg.RKPCA_PLM}. At $t$-th iteration, we linearize $\textup{Tr}(\bm{K}^{1/2})$ at $\bm{E}^{(t-1)}$ as
\begin{equation}
\begin{aligned}
\textup{Tr}(\bm{K}^{1/2})\approx\ &\textup{Tr}(\bm{K}_{t-1}^{1/2})+\langle \dfrac{\partial \mathcal{J}}{\partial \bm{E}^{(t-1)}}, \bm{E}-\bm{E}^{(t-1)}\rangle\\
&+\dfrac{\nu}{2}\Vert \bm{E}-\bm{E}^{(t-1)}\Vert_F^2,
\end{aligned}
\end{equation}
where $\nu$ is the step size and
\begin{equation}\label{Eq.Al3_dE}
\dfrac{\partial \mathcal{J}}{\partial \bm{E}^{(t-1)}}=-\dfrac{2}{\sigma^2}((\bm{M}-\bm{E}^{(t-1)})\bm{H}-(\bm{M}-\bm{E}^{(t-1)})\odot (\bm{B}\bm{H})).
\end{equation}
Then we solve
\begin{equation}
\min\limits_{\bm{E}}\dfrac{\nu}{2}\Vert \bm{E}-\bm{E}^{(t-1)}+\dfrac{\partial \mathcal{J}}{\partial \bm{E}^{(t-1)}}/\nu\Vert_F^2+\lambda\Vert \bm{E}\Vert_1,
\end{equation}
for which the solution is
\begin{equation}\label{Eq.solve_E_result}
\bm{E}^{(t)}=\Theta_{\lambda/\nu}(\bm{E}^{(t-1)}-\dfrac{\partial \mathcal{J}}{\partial \bm{E}^{(t-1)}}/\nu).
\end{equation}
The Lipschitz constant of $\dfrac{\partial \mathcal{J}}{\partial \bm{E}^{(t-1)}}$ is a natural choice of the step size $\nu$. Similar to (\ref{Eq.estL}), we can estimate the Lipschitz constant of $\dfrac{\partial \mathcal{J}}{\partial \bm{E}^{(t-1)}}$ as 
\begin{equation}
\hat{L}_{\triangledown\mathcal{J}}=\Vert\dfrac{2}{\sigma^2}(\bm{H}-\varsigma\bm{I})\Vert_2.
\end{equation}
In Algorithm \ref{alg.RKPCA_PLM}, we set $\nu$ as $\omega \hat{L}_{\triangledown\mathcal{J}}$ and increase $\omega$ by $\omega=c\omega$ if the objective function does not decrease. The convergence condition of Algorithm \ref{alg.RKPCA_PLM} is $\Vert \bm{E}^{(t)}-\bm{E}^{(t-1)}\Vert_F/\Vert \bm{M}\Vert_F<\varepsilon$ (e.g.$10^{-4}$). We have the following lemma:
\begin{theorem}\label{lem0}
(a) The objective function in Algorithm \ref{alg.RKPCA_PLM} is able to converge if $t$ is large enough. (b) The solution of $\bm{E}$ generated by Algorithm \ref{alg.RKPCA_PLM} is able to converge if $t$ is large enough.
\end{theorem}
The above lemma will be proved in the following content.

\renewcommand{\algorithmicrequire}{\textbf{Input:}}
\renewcommand{\algorithmicensure}{\textbf{Output:}}
\begin{algorithm}[h]
\caption{Solve RKPCA with PLM+AdSS}
\label{alg.RKPCA_PLM}
\begin{algorithmic}[1] 
\Require
$\bm{M}$, $k(\cdot,\cdot)$, $\lambda$, $\omega=0.1$, $c>1$, $t_{max}$.
\State \textbf{initialize:} $\bm{E}^{(0)}=0$, $t=0$
\Repeat
\State Compute $\dfrac{\partial \mathcal{J}}{\partial \bm{E}^{(t-1)}}$ by (\ref{Eq.Al3_dE})
\State $\nu=\omega\Vert\dfrac{2}{\sigma^2}(\bm{H}-\varsigma\bm{I})\Vert_2$
\State $\bm{E}^{(t)}=\Theta_{\lambda/\nu}(\bm{E}^{(t-1)}-\dfrac{\partial \mathcal{J}}{\partial \bm{E}^{(t-1)}}/\nu)$
\If{$\mathcal{J}(\bm{E}^{(t)})>\mathcal{J}(\bm{E}^{(t-1)})$}
$\omega=c\omega$
\EndIf
\State $t=t+1$
\Until{converged or $t=t_{max}$}
\Ensure $\bm{E}=\bm{E}^{(t)}$, $\bm{X}=\bm{M}-\bm{E}$
\end{algorithmic}
\end{algorithm}

The problem of RKPCA given by (\ref{Eq.RKPCA_3}) is a case of the following general problem
\begin{equation}
\min\limits_{\bm{\upsilon}} J(\bm{\upsilon})=f(\bm{\upsilon})+g(\bm{\upsilon}),
\end{equation}
where $f(\cdot)$ is nonconvex but differentiable and $g(\cdot)$ is convex but not differentiable. We linearize $f(\bm{\upsilon})$ at $\bar{\bm{\upsilon}}$ as
\begin{equation}
f(\bm{\upsilon})=f(\bar{\bm{\upsilon}})+\langle \bm{\upsilon}- \bar{\bm{\upsilon}}, \triangledown f(\bar{\bm{\upsilon}})\rangle+\dfrac{\tau}{2}\Vert \bm{\upsilon}-\bar{\bm{\upsilon}}\Vert^2+e,
\end{equation}
where $\triangledown f(\bar{\bm{\upsilon}})$ denotes the gradient of $f(\cdot)$ at $\bar{\bm{\upsilon}}$ and $e$ denotes the residual of the quadratic approximation. Then we solve
\begin{equation}
\min\limits_{\bm{\upsilon}} g(\bm{\upsilon})+\langle \bm{\upsilon}- \bar{\bm{\upsilon}}, \triangledown f(\bar{\bm{\upsilon}})\rangle+\dfrac{\tau}{2}\Vert \bm{\upsilon}-\bar{\bm{\upsilon}}\Vert^2.
\end{equation}
The closed-form solution is obtained by the proximal algorithm \cite{ProximalA}, i.e.,
\begin{equation}
\upsilon^+\in prox_{\tau}^g(\bm{\upsilon}-\triangledown f(\bar{\bm{\upsilon}})/\tau).
\end{equation}
We give the following two lemmas.
\begin{lemma}\label{lem1}
If $f(\cdot)$ is continuously differentiable and its gradient $\triangledown f$ is $L_{\triangledown f}$-Lipschitz continuous, then
\begin{equation}\label{Eq.LipC}
f(\bm{\upsilon})\leq f(\bar{\bm{\upsilon}})+\langle \upsilon-\bar{\bm{\upsilon}},\triangledown f(\bar{\bm{\upsilon}})\rangle+\dfrac{L_{\triangledown f}}{2}\Vert \bm{\upsilon}-\bar{\bm{\upsilon}}\Vert^2.
\end{equation}
\end{lemma}
\begin{lemma}\label{lem2}
Given that the gradient of $f(\cdot)$ is $L_{\triangledown f}$-Lipschitz continuous, $\bm{\upsilon}^+\in prox_{\tau}^g(\bm{\upsilon}-\triangledown f(\bar{\bm{\upsilon}})/\tau)$, and $\tau>L_{\triangledown f}$, we have
\begin{equation}
f(\bm{\upsilon}^+)+g(\bm{\upsilon}^+)\leq f(\bar{\bm{\upsilon}})+g(\bar{\bm{\upsilon}})-\dfrac{\tau-L_{\triangledown f}}{2}\Vert \bm{\upsilon}^+-\bm{\upsilon}\Vert^2.
\end{equation}
\end{lemma}
\begin{proof}
As $\bm{\upsilon}^+\in \min\limits_{\bm{\upsilon}} g(\bm{\upsilon})+\langle \bm{\upsilon}- \bar{\bm{\upsilon}}, \triangledown f(\bar{\bm{\upsilon}})\rangle+\dfrac{\tau}{2}\Vert \bm{\upsilon}-\bar{\bm{\upsilon}}\Vert^2$, by taking $\bm{\upsilon}=\bar{\bm{\upsilon}}$, we have
\begin{equation}\label{Eq.Lem2P1}
g(\bm{\upsilon}^+)+\langle \bm{\upsilon}^+- \bar{\bm{\upsilon}}, \triangledown f(\bar{\bm{\upsilon}})\rangle+\dfrac{\tau}{2}\Vert \bm{\upsilon}^+-\bar{\bm{\upsilon}}\Vert^2\leq g(\bar{\bm{\upsilon}}).
\end{equation}
Using Lemma \ref{lem1} with $\bm{\upsilon}=\bm{\upsilon}^+$, we have
\begin{equation}\label{Eq.Lem2P2}
f(\bm{\upsilon}^+)\leq f(\bar{\bm{\upsilon}})+\langle \bm{\upsilon}^+-\bar{\bm{\upsilon}},\triangledown f(\bar{\bm{\upsilon}})\rangle+\dfrac{L_{\triangledown f}}{2}\Vert \bm{\upsilon}^+-\bar{\bm{\upsilon}}\Vert^2.
\end{equation}
Adding (\ref{Eq.Lem2P2}) with (\ref{Eq.Lem2P1}), we have
\begin{equation}
f(\bm{\upsilon}^+)+g(\bm{\upsilon}^+)\leq f(\bar{\bm{\upsilon}})+g(\bar{\bm{\upsilon}})-\dfrac{\tau-L_{\triangledown f}}{2}\Vert \bm{\upsilon}^+-\bm{\upsilon}\Vert^2.
\end{equation}
This finished the proof of Lemma \ref{lem2}.
\end{proof}

Lemma \ref{lem2} shows that when $\tau>L_{\triangledown f}$,
\begin{equation}
J(\bm{\upsilon}^{(t+1)})<J(\bm{\upsilon}^{(t)})<\cdots<J(\bm{\upsilon}^{(1)})<J(\bm{\upsilon}^{(0)}).
\end{equation}
where $\bm{\upsilon}^{(1)},\cdots,\bm{\upsilon}^{(t+1)}$ is a series of $\bm{\upsilon}$ obtained from the proximal linearized algorithm.
Since $J(\bm{\upsilon})>-\infty$, the proximal linearized algorithm is able to converge. 

In Algorithm \ref{alg.RKPCA_PLM}, $\tau$ is estimated as
$\nu=\omega\hat{L}_{\triangledown\mathcal{J}}=\omega\Vert 2(\bm{H}-\varsigma\bm{I})/\sigma^2\Vert_2$. Let $L_{\triangledown\mathcal{J}}^{(t)}$ be the true Lipschitz constant of $\dfrac{\partial \mathcal{J}}{\partial \bm{E}^{(t)}}$ at $t$-th iteration. Let $\nu^{(t)}$ be the estimated $\tau$ at $t$-th iteration. We have
\begin{equation}\label{Eq.Jt}
\mathcal{J}(\bm{E}^{(t)})-\mathcal{J}(\bm{E}^{(t-1)})\leq -\dfrac{\nu^{(t)}-L_{\triangledown\mathcal{J}}^{(t)}}{2}\Vert \bm{E}^{(t)}-\bm{E}^{(t-1)}\Vert_F^2.
\end{equation}
We increase $\omega$ by $\omega=c\omega$ if $\mathcal{J}(\bm{E}^{(t)})-\mathcal{J}(\bm{E}^{(t-1)})>0$. If $\omega^{(t)}>L_{\triangledown\mathcal{J}}^{(t)}/\hat{L}_{\triangledown\mathcal{J}}^{(t)}$, we have $\mathcal{J}(\bm{E}^{(t)})<\mathcal{J}(\bm{E}^{(t-1)})$. If $c$ is large enough, there exists a number $l>1$ such that
\begin{equation}
\nu^{(t+1)}>L_{\triangledown\mathcal{J}}^{(t+1)},\nu^{(t+2)}>L_{\triangledown\mathcal{J}}^{(t+2)},\cdots,\nu^{(t+l)}>L_{\triangledown\mathcal{J}}^{(t+l)}
\end{equation}
where $\omega^{(t+l)}=\cdots=\omega^{(t+1)}=\omega^{(t)}$.
Therefore, $\mathcal{J}(\bm{E})$ will decrease with high probability even when $\omega$ is relatively small. When $\omega$ is large enough, we will always have $\nu>L_{\triangledown\mathcal{J}}$. Since $\mathcal{J}(\bm{E})$ is bounded, when $t\rightarrow\infty$, $\mathcal{J}(\bm{E}^{(t)})-\mathcal{J}(\bm{E}^{(t)})=0$. It means that the objective function in Algorithm \ref{alg.RKPCA_PLM} is able to converge. This proved Theorem \ref{lem0}(a).

Summing both sides of (\ref{Eq.Jt}) from $1$ to $N$, we have
\begin{equation}
\mathcal{J}(\bm{E}^{(0)})-\mathcal{J}(\bm{E}^{(N)})\geq \sum_{t=1}^N \dfrac{\nu^{(t)}-L_{\triangledown\mathcal{J}}^{(t)}}{2}\Vert \bm{E}^{(t)}-\bm{E}^{(t-1)}\Vert_F^2,
\end{equation}
which indicates that
\begin{equation}
\sum_{t=1}^{\infty} \dfrac{\nu^{(t)}-L_{\triangledown\mathcal{J}}^{(t)}}{2}\Vert \bm{E}^{(t)}-\bm{E}^{(t-1)}\Vert_F^2<\infty.
\end{equation}
When $t$ is small, $\nu^{(t)}>L_{\triangledown\mathcal{J}}^{(t)}$ holds with high probability. If $t$ is larger than $\bar{t}$ (when $\omega$ reaches to a large enough value), $\nu^{(t)}>L_{\triangledown\mathcal{J}}^{(t)}$ always holds. Then we have
\begin{equation}
s+\sum_{t=\bar{t}}^{\infty} \dfrac{\nu^{(t)}-L_{\triangledown\mathcal{J}}^{(t)}}{2}\Vert \bm{E}^{(t)}-\bm{E}^{(t-1)}\Vert_F^2<\infty,
\end{equation}
where $s=\sum_{t=1}^{\bar{t}} \dfrac{\nu^{(t)}-L_{\triangledown\mathcal{J}}^{(t)}}{2}\Vert \bm{E}^{(t)}-\bm{E}^{(t-1)}\Vert_F^2\geq 0$. It further indicates $\bm{E}^{(t)}-\bm{E}^{(t-1)}=0$ if $t\rightarrow \infty$. Therefore, the solution of $\bm{E}$ generated by Algorithm \ref{alg.RKPCA_PLM} is able to converge if $t$ is large enough. This proved Theorem \ref{lem0}(b).

Compared with Algorithm \ref{alg.RKPCA}, Algorithm \ref{alg.RKPCA_PLM} cannot ensure that the objective function is always decreasing when $t$ is small. In Algorithm \ref{alg.LineSearch}, to find a suitable $\eta$, the backtracking line search requires evaluating the objective function multiple times, which will increase the computational cost. These two algorithm will be compared in the section of experiments.

The main computational cost of RKPCA is from the computation of $\bm{K}^{p/2}$ and its gradient, which requires performing SVD on an $n\times n$ matrix. Therefore, the computational complexity of RKPCA is $O(n^3)$. In RPCA, the main computational cost is from the SVD on a $d\times n$ matrix. Hence, the computational complexity of RPCA is $O(\min(d^2n,dn^2))$. Although the computational cost of RKPCA is higher than that of RPCA, the recovery accuracy of RKPCA is much higher than that of RPCA for nonlinear data and high-rank matrices. For large-scale data (e.g., $n>2000$), instead of full SVD, truncated SVD or randomized SVD (RSVD) \cite{randomsvd} should be applied to RKPCA. The computational complexity of RKPCA with RSVD in each iteration is reduced to $O(\tilde{r}n^2)$, where $\tilde{r}$ is the approximate rank of $\bm{K}$ at iteration $t$ and is non-increasing. RSVD \cite{randomsvd} and the corresponding optimization of RKPCA (solved by PLM+AdSS) are shown in Algorithm \ref{alg.rsvd} and Algorithm \ref{alg.RKPCA_rsvd} respectively. 

\renewcommand{\algorithmicrequire}{\textbf{Input:}}
\renewcommand{\algorithmicensure}{\textbf{Output:}}
\begin{algorithm}[h]
\caption{Randomized Singular Value Decomposition \cite{randomsvd}}
\label{alg.rsvd}
\begin{algorithmic}[1] 
\Require
$\bm{X}\in\mathbb{R}^{m\times n}$, $\tilde{r}$.
\State Generate random Gaussian matrix $\bm{P}\in\mathbb{R}^{n\times 2\tilde{r}}$
\State $\bm{Y}=\bm{X}\bm{P}$
\State $\bm{W}\longleftarrow$ orthonormal basis for the range of $\bm{Y}$
\State $\bm{Z}=\bm{W}^T\bm{X}$
\State Perform economy SVD: $\bm{Z}=\bm{U}\bm{S}\bm{V}^T$
\State $\bm{U}\longleftarrow\bm{W}\bm{U}$
\State $\bm{U}\longleftarrow\lbrace\bm{U}_{\cdot j}\rbrace_{j=1}^{\tilde{r}},\bm{S}\longleftarrow\lbrace\bm{S}_{jj}\rbrace_{j=1}^{\tilde{r}},\bm{V}\longleftarrow\lbrace\bm{V}_{\cdot j}\rbrace_{j=1}^{\tilde{r}}$
\Ensure $\bm{U},\bm{S},\bm{V}$ ($\bm{X}\approx\bm{U}\bm{S}\bm{V}^T)$
\end{algorithmic}
\end{algorithm}

\renewcommand{\algorithmicrequire}{\textbf{Input:}}
\renewcommand{\algorithmicensure}{\textbf{Output:}}
\begin{algorithm}[h]
\caption{RKPCA solved by PLM+AdSS and RSVD}
\label{alg.RKPCA_rsvd}
\begin{algorithmic}[1] 
\Require
$\bm{M}$, $k(\cdot,\cdot)$, $\lambda$, $p$, $\tilde{r}$, $\xi$, $\omega=0.1$, $c>1$, $t_{max}$.
\State \textbf{initialize:} $\bm{E}^{(0)}=0$, $t=0$
\Repeat
\State $\bm{K}=\bm{U}\bm{S}\bm{V}^T$ using Algorithm \ref{alg.rsvd} with parameter $\tilde{r}$ 
\State $\tilde{r}\longleftarrow$ the number of sigular values larger than $\xi$
\State $\bm{K}^{\tfrac{p}{2}-1}=\bm{U}\bm{S}^{\tfrac{p}{2}-1}\bm{V}^T$
\State $\bm{H}=\dfrac{p}{2}\bm{K}^{\tfrac{p}{2}-1}\odot\bm{K}$
\State Compute $\dfrac{\partial \mathcal{J}}{\partial \bm{E}^{(t-1)}}$ using (\ref{Eq.Al3_dE})
\State $\nu=\omega\Vert\dfrac{2}{\sigma^2}(\bm{H}-\varsigma\bm{I})\Vert_2$
\State $\bm{E}^{(t)}=\Theta_{\lambda/\nu}(\bm{E}^{(t-1)}-\dfrac{\partial \mathcal{J}}{\partial \bm{E}^{(t-1)}}/\nu)$
\If{$\mathcal{J}(\bm{E}^{(t)})>\mathcal{J}(\bm{E}^{(t-1)})$}
$\omega=c\omega$
\EndIf
\State $t=t+1$
\Until{converged or $t=t_{max}$}
\Ensure $\bm{E}=\bm{E}^{(t)}$, $\bm{X}=\bm{M}-\bm{E}$
\end{algorithmic}
\end{algorithm}

\subsection{Robust subspace clustering by RKPCA}
Sparse subspace clustering \cite{SSC2009} is an effective and popular method to cluster data into different groups according to different subspaces. In recent years, a few robust variants of SSC and other methods have been proposed for handling noises and outliers in subspace clutering \cite{SSC_PAMIN_2013,5995365,soltanolkotabi2014,7159061,7222444,NoisySSCJMLR,SRSSC}. For example, in \cite{SSC_PAMIN_2013}, the modified version of SSC is robust to noise and sparse outlying entries. In \cite{soltanolkotabi2014} and \cite{NoisySSCJMLR}, the variants of SSC are provably effective in handling noises. In \cite{7159061}, a variant of SSC incorporating the robust measure correntropy was proposed to deal with large corruptions. In \cite{SRSSC}, through using random anchor points and multilayer graphs, a scalable and robust variant of SSC was proposed. Nevertheless, the clustering performances of these methods are unsatisfactory when the data are heavily corrupted by sparse noises. The reason is that the highly corrupted data can make the sparse or low-rank representations significantly biased.

To improve the clustering accuracy, we can use RPCA or RKPCA to remove the noises and then perform SSC or other methods on the clean data. As a matter of fact, RPCA and RKPCA can be directly applied to subspace clustering, even when the data are heavily corrupted. Given a noisy data matrix $\bm{M}$, we use RKPCA to recover the clean data matrix $\bm{X}$. The truncated singular value decomposition of the matrix in the nonlinear feature space is given by
\begin{equation}
\phi(\bm{X})\approx \bm{U}_r\bm{S}_r\bm{V}_r^T,
\end{equation}
where $r$ is the latent dimensionality of the feature space. The subspaces of $\phi(\bm{X})$ can be segmented by using $\bm{V}_r$. We propose Algorithm \ref{alg.sc_rkpca} to cluster $\phi(\bm{X})$ or $\bm{X}$ according to the subspace membership. Similar to that in LRR \cite{LRR_PAMI_2013,NLRR_2016}, the parameter $q$ is an even number (e.g. 4) to ensure the non-negativity and sparsity of the affinity matrix $\bm{A}$.

\renewcommand{\algorithmicrequire}{\textbf{Input:}}
\renewcommand{\algorithmicensure}{\textbf{Output:}}
\begin{algorithm}[h]
\caption{Subspace clustering by RKPCA}
\label{alg.sc_rkpca}
\begin{algorithmic}[1] 
\Require
$\bm{M}\in\mathbb{R}^{d\times n}$, $C$, $r$, $q$.
\State Compute $\bm{X}$ and the corresponding $\bm{K}$ using Algorithm \ref{alg.RKPCA},  Algorithm \ref{alg.RKPCA_PLM}, or Algorithm \ref{alg.RKPCA_rsvd}
\State Perform eigenvalue decomposition $\bm{K}\approx \bm{V}_r\bm{S}_r\bm{V}_r^T$
\State $[\bm{V}_r]_{i\cdot}=[\bm{V}_r]_{i\cdot}/\Vert [\bm{V}_r]_{i\cdot}\Vert_2$, $i=1,2,\cdots,n$
\State $\bm{A}_{ij}=[\bm{V}_r\bm{V}_r^T]_{ij}^q$, $\bm{A}_{ij}=0$, $i,j=1,2,\cdots,n$
\State Perform spectral clustering on $\bm{A}$
\Ensure $C$ clusters of $\bm{X}$ or $\bm{M}$
\end{algorithmic}
\end{algorithm}

\section{Experiments}
\subsection{Synthetic data}\label{sec.syn}
\subsubsection{Data generation}
We generate a nonlinear data matrix by
\begin{equation}\label{Eq.Syn1}
\bm{X}=\bm{P}_1\bm{Z}+0.5(\bm{P}_2\bm{Z}^{\odot 2}+\bm{P}_3\bm{Z}^{\odot 3})
\end{equation}
where $\bm{Z}\in \mathbb{R}^{r\times n}$ are uniformly drawn from the interval $(-1,1)$, $\bm{P}\in\mathbb{R}^{d\times r}$ are drawn from standard normal distribution, and $\bm{Z}^{\odot u}$ denotes the $u$-th power acted on each entry of $\bm{Z}$. We set $d=20$, $r=2$, and $n=100$. The model (\ref{Eq.Syn1}) maps the low-dimensional latent variable $\bm{z}$ to high-dimensional variable $\bm{x}$ through a nonlinear function $\bm{x}=f(\bm{z})$. The nonlinearity of the data is not very strong, which is quite practical because linearity and nonlinearity always exist simultaneously. We added Gaussian noises $\mathcal{N}(0,1)$ to certain percentage (noise density, denoted by $\rho$) of the entries in $\bm{X}$ and then get a matrix $\bm{M}$ with sparse noises. We increase the noise density from $10\%$ to $80\%$ and use PCA, KPCA \cite{PreImageKPCA2004TNN}, RPCA \cite{RPCA} and RKPCA to recover $\bm{X}$ from $\bm{M}$. The recovery performance is measured by the relative error defined as 
\begin{equation}\label{Eq.e_rlt}
e_{rlt}=\Vert \bm{X}-\hat{\bm{X}}\Vert_F/\Vert \bm{X}\Vert_F,
\end{equation}
where $\hat{\bm{X}}$ denotes the recovered matrix. In KPCA and RKPCA, the parameter $\sigma$ of RBF kernel is determined by (\ref{Eq.kerpar}) with $\beta=1$. In fact, we have tried to extend the robust KPCA proposed in \cite{nguyen2009robust} to an unsupervised version that does not require clean training data. However, the performance is not satisfactory and absolutely not comparable to that of KPCA preimage proposed in \cite{PreImageKPCA2004TNN}. 

\subsubsection{Iteration performances of ADMM+BTLS and PLM+AdSS}
Figure \ref{Fig.converge_curve} shows the iteration performances of Algorithm \ref{alg.RKPCA} (ADMM+BTLS) and Algorithm \ref{alg.RKPCA_PLM} (PLM+AdSS) for solving RKPCA on the matrix with $50\%$ noise density. It is found that, the two algorithms can work well with a large range of the parameters. In ADMM+BTLS, $\eta=10L_{\mathcal{J}}$ and $\eta=5L_{\mathcal{J}}$ achieve nearly the same results because the true Lipschitz constant may be around $5L_{\mathcal{J}}$. In PLM+AdSS, though $\omega=0.01$ make the optimization diverged at the beginning of iteration, the algorithm still converged in about 60 iterations. In general, although ADMM+BTLS is able to reduce the objective function (augmented Lagrange function) in every iteration, its convergence speed is lower than that of PLM+AdSS. In addition, the outputs of the two algorithms are nearly the same. Therefore, PLM+AdSS is preferable to ADMM+BTLS. In this paper, RKPCA is solved by PLM+AdSS with $\omega=0.1$.

\begin{figure}[tb]
\centering
\includegraphics[width=7cm]{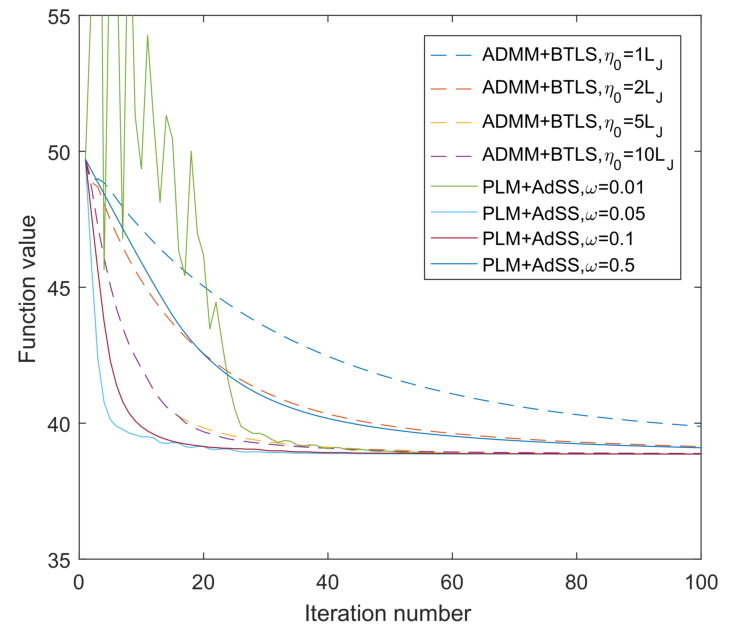}
\caption{Iteration performances of ADMM+BTLS and PLM+AdSS (for the first problem in Section \ref{sec.syn} with $50\%$ noise density)}\label{Fig.converge_curve}
\end{figure}

\subsubsection{Recovery results}
The experiments are repeated for 100 times and the average results are reported in Table \ref{Tab.Syn1}. As can be seen, the recovery error of RKPCA is significantly lower than those of other methods in almost all cases. According to Theorem \ref{theorem_sufficient}, to approximately recover the matrix, the noise density should meet $\rho<76\%$, which matches the results in Table \ref{Tab.Syn1} because the recovery error for $\rho=80\%$ is significantly higher than that for $\rho=70\%$. We use (\ref{Eq.Syn1}) with $d=20$, $r=2$, and $n=50$ to generate 5 matrices and stack them together to form a matrix of size $20\times 250$. Hence the matrix consists of the nonlinear data drawn from 5 different subspaces. We add Gaussian noises $\mathcal{N}(0,1)$ to $\rho$ (from $10\%$ to $50\%$) percentage of the entries of the matrix. The recovery errors (average of 50 trials) are reported in Table \ref{Tab.Syn2}. RKPCA consistently outperforms other methods.

\begin{table}[h!]
\caption{Relative errors ($\%$) on single-subspace data}
\centering
\begin{tabular}{cccccc}
\hline
     $\rho$      &      noisy &        PCA &       KPCA &       RPCA &      RKPCA \\
\hline
      10\% &     33.83  &     19.55  &     18.18  &  	{\bf 2.57}   &   2.62  \\

      20\% &     49.59  &     26.63  &     25.73  &     10.56  & {\bf 4.93 } \\

      30\% &     63.07  &     32.36  &     33.01  &     19.06  & {\bf 10.56 } \\

      40\% &     66.35  &     32.83  &     34.48  &     23.97  & {\bf 15.44 } \\

      50\% &     81.95  &     40.02  &     41.69  &     32.54  & {\bf 24.18 } \\

      60\% &     84.25  &     40.65  &     43.62  &     37.72  & {\bf 27.61 } \\

      70\% &     93.74  &     45.26  &     47.85  &     44.38  & {\bf 34.92 } \\
      
      80\% &     98.21  &     49.69  &     51.85  &     49.45  & {\bf 44.23 } \\
\hline
\end{tabular} \label{Tab.Syn1}
\end{table}

\begin{table}[htb]
\vspace{-0.2cm}
\caption{Relative errors ($\%$) on multiple-subspace data}
\footnotesize
\centering
\begin{tabular}{cccccc}
\hline
   $\rho$        &        noisy &        PCA &       KPCA &       RPCA &      RKPCA \\
\hline
      10\% &     35.22  &     31.84  &     25.82  &     26.27  & {\bf 9.88 } \\

      20\% &     49.86  &     42.84  &     35.52  &     36.75  & {\bf 19.6 } \\

      30\% &     60.97  &     50.39  &     42.41  &     44.86  & {\bf 29.07 } \\

      40\% &     68.49  &     55.29  &     47.39  &     49.95  & {\bf 36.16 } \\

      50\% &     75.35  &     58.62  &     51.89  &     56.81  & {\bf 44.62 } \\
\hline
\end{tabular}  \label{Tab.Syn2}
\end{table}

We also study the influence of the parameter $p$ in RKPCA. Figure \ref{Fig.pvalue}(a) and Figure \ref{Fig.pvalue}(b) show the recovery errors of RKPCA with different $p$ ($0.1\leq p\leq 1$) on the single-subspace data and multiple-subspace data. The recovery errors of PCA, KPCA, and RPCA are also shown in the figure. It can be seen that RKPCA with a smaller $p$ around $0.6$ can outperform RKPCA with $p=1$. However, for simplicity, throughout this paper, we only report the result of RKPCA with $p=1$ because it has already shown significant improvement compared with other methods such as RPCA.
\begin{figure}[h!]
\centering
\includegraphics[width=8.5cm]{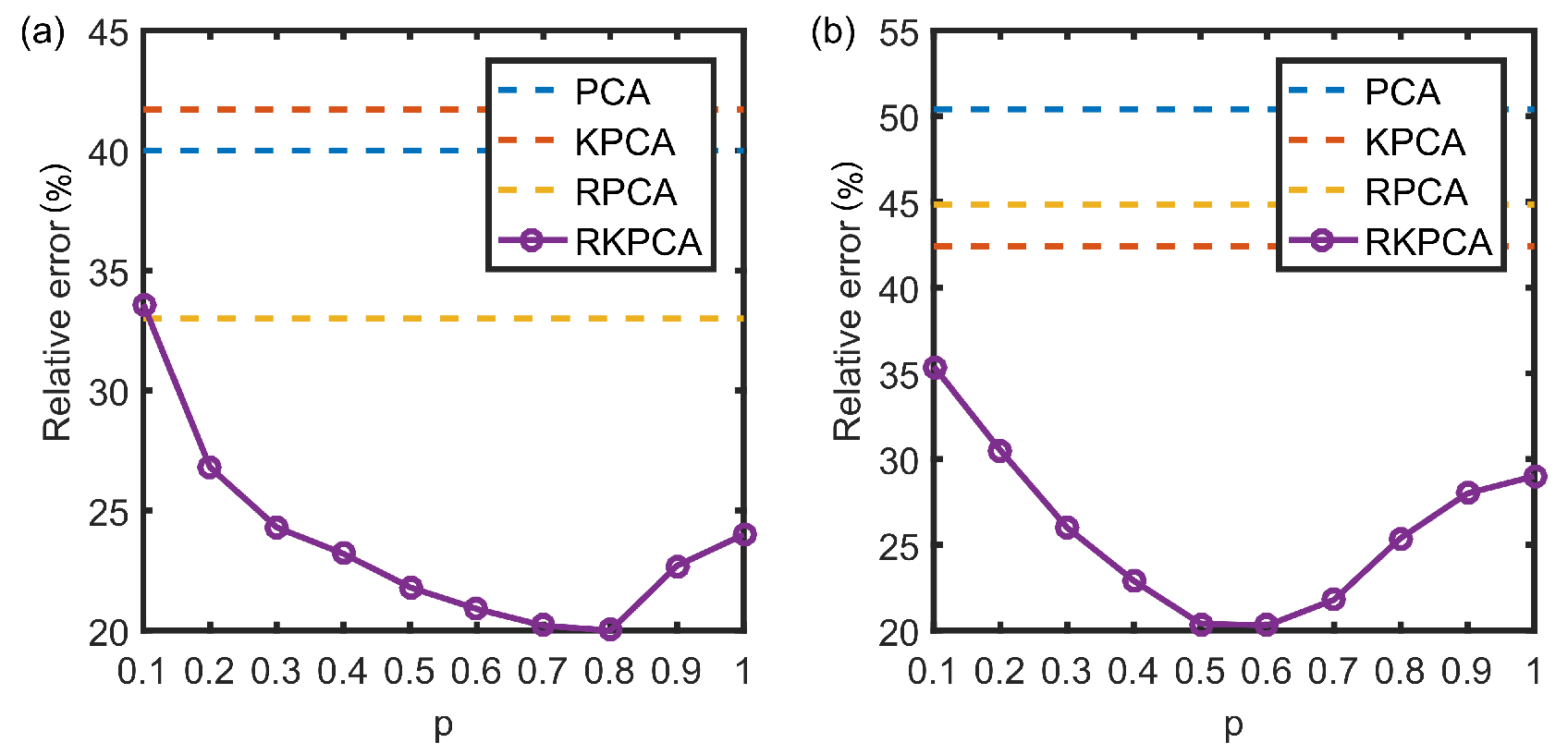}
\caption{RKPCA with different $p$: (a) single-subspace data ($\rho=50\%$); (b) multiple-subspace data ($\rho=30\%$).}\label{Fig.pvalue}
\end{figure}

\subsection{Natural images}
\subsubsection{Datasets} 
We use four datasets of natural images to evaluate the proposed method RKPCA. The details are as follows.
\begin{itemize}  
\item \textbf{MNIST}\cite{Data_MNIST1998} The dataset consists of the handwritten digits $0\sim 9$. We use a subset consisting of 1000 images (100 images for each digit) and resize all images to $20\times 20$.
\item \textbf{ORL}\cite{ORL_face} The datasets consists of the face images of 40 subjects. Each subject has 10 images with different poses and facial expressions. The original size of each image is $112\times 92$. We resize the images to $32\times 28$.
\item \textbf{COIL20}\cite{Dataset_COIL20} The datasets consists of the images of 20 objects. Each object has 72 images of different poses. We resize the images to $32\times 32$.
\item \textbf{YaleB+}\cite{Dataset_ExtendYaleB} The Extended Yale face dataset B consists of the face images of 38 subjects. Each subject has about 64 images under various illumination conditions. We resize the images to $32\times 32$.
\end{itemize}
For each of the four datasets, we consider two noisy cases. In the first case (pixel-noisy), we add salt and pepper noise of $30\%$ density to all the images. In the second case (block-noisy), we mask the images at random positions with a block of $20\%$ width and height of the image. Since the backgrounds of the MNIST and COIL20 images are black, the value of the block mask is $1$ while the value of the block mask for ORL and YaleB+ images is $0$. Figure \ref{Fig.Images} shows a few examples of the original images and noisy images.

\begin{figure}[h!]
\centering
\includegraphics[width=8cm]{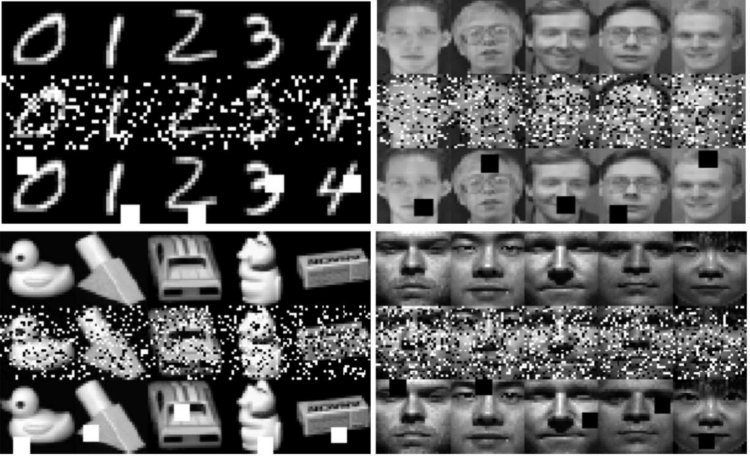}
\caption{A few samples of original images and noisy images (groups $1\sim4$: MNIST, ORL, COIL20, YaleB+; in each group, rows $1\sim3$: original, pixel-noisy, and block-noisy images)}\label{Fig.Images}
\end{figure}

\subsubsection{Noise removal}
We compare RKPCA with PCA, KPCA \cite{PreImageKPCA2004TNN}, and RPCA \cite{RPCA} in recovering the original images of the four datasets. The parameter $\sigma$ of RBF kernel in KPCA and RKPCA is determined by (\ref{Eq.kerpar}) with $\beta=1.5$. Since the original images especially those of MNIST and YaleB+ are noisy, we use one more evaluation metric, the generalization error of k-nearest neighbors. It is defined by
\begin{equation}
e_{knn}=\dfrac{1}{n}\sum_{i=1}^n y_k(i),
\end{equation}
where $y_k(i)=1$ if the label of sample $i$ is not the same as the most frequent label of its $k$ nearest neighbors and $y_k(i)=0$ otherwise. A small value of $e_{knn}$ indicates the intra-class difference is small compared to the inter-class difference. The average results of the relative error and 5nn error of 10 repeated trials are reported in Table \ref{Tab.real1}. The relative recovery error and 5nn generalization error of RKPCA are always smaller than those of other methods. A few recovered images given by RPCA and RKPCA are shown in Figure \ref{Fig.Images_re}. The recovery performances of RKPCA are consistently better than those of RPCA.

\begin{table}[h!]
\caption{Relative recovery error ($\%$) and 5nn generalization error ($\%$) of noise removal on natural images}
\centering
\begin{tabular}{lllllll}
\hline
   Dataset & Case &     Metric &        PCA &       KPCA &       RPCA &      RKPCA \\
\hline
\multirow{ 4}{*}{MNIST} & \multirow{2}{*}{pixel-noisy} &      $e_{rlt}$ &       75.8 &      69.27 &  53.61 &      {\bf 48.13} \\

 &  &      $e_{5nn}$ &       31.6 &       27.6 &       18.7 & {\bf 16.7} \\

 & \multirow{ 2}{*}{block-noisy} &      $e_{rlt}$ &      69.97 &      63.46 &      56.32 & {\bf 47.02} \\    

 &  &      $e_{5nn}$ &         39 &         38 &       24.3 & {\bf 19.9} \\
\hline
\multirow{ 4}{*}{ORL} & \multirow{ 2}{*}{pixel-noisy} &      $e_{rlt}$ &      32.06 &      29.39 &      13.85 & {\bf 12.93} \\

 &  &      $e_{5nn}$ &      60.75 &      49.25 &       6.75 & {\bf 5.75} \\

 & \multirow{ 2}{*}{block-noisy} &      $e_{rlt}$ &     26.75  &     25.11  &     17.05  & {\bf 11.23 } \\

 &  &      $e_{5nn}$ &     42.00  &     39.50  &     15.50  & {\bf 8.25 } \\
\hline
\multirow{ 4}{*}{COIL20} & \multirow{ 2}{*}{pixel-noisy} &      $e_{rlt}$ &     45.73  &     39.04  &     16.57  & {\bf 15.78} \\

 &  &      $e_{5nn}$ &       5.83 &       5.66 &       0.76 & {\bf 0.56} \\

 & \multirow{ 2}{*}{block-noisy} &      $e_{rlt}$ &      34.65 &      31.28 &       18.2 & {\bf 14.3} \\

 &  &      $e_{5nn}$ &       9.37 &       9.16 &       3.96 & {\bf 1.46} \\
\hline
\multirow{ 4}{*}{YaleB+} & \multirow{ 2}{*}{pixel-noisy} &      $e_{rlt}$ &      66.52 &      58.91 &      18.36 & {\bf 15.31} \\

 &  &      $e_{5nn}$ &      74.68 &      75.68 &  53.32 &      {\bf 46.72} \\

& \multirow{ 2}{*}{block-noisy} &      $e_{rlt}$ &      25.41 &      22.91 &      18.04 & {\bf 13.96} \\

&  &      $e_{5nn}$ &      63.62 &      53.11 &      51.53 & {\bf 42.21} \\
\hline
\end{tabular} \label{Tab.real1}
\end{table}

\begin{figure}[h!]
\centering
\includegraphics[width=8cm]{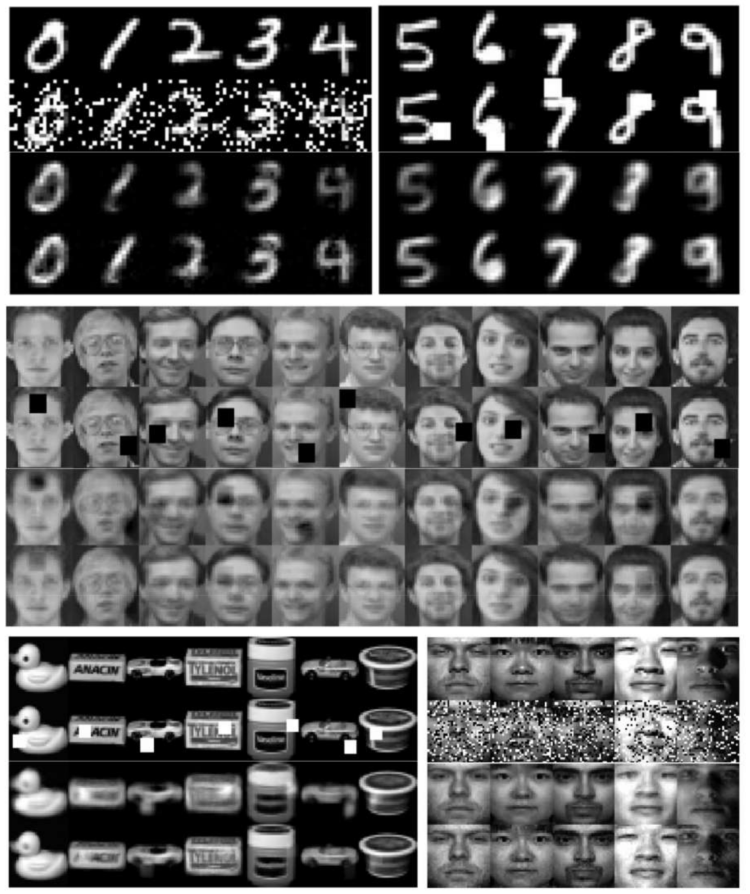}
\caption{Images recovered by RPCA and RKPCA (in each group, rows $1\sim4$ are original, noisy, RPCA-recovred and RKPCA-recovered images)}\label{Fig.Images_re}
\end{figure}

\subsubsection{Subspace clustering} 
We compare RKPCA with SSC \cite{SSC_PAMIN_2013}, NLRR \cite{NLRR_2016}, RKLRS\cite{xiao2016robust}, and RPCA \cite{RPCA} in subspace clustering on the four datasets. Particularly, we normalize all image vectors of YaleB+ dataset to have unit $\ell_2$ norm because many images have very low brightness. The numbers ($r$) of singular vectors for constructing the affinity matrices in NLRR, RPCA, and RKPCA on each dataset are as follows: MNIST, $r_{NLRR}=r_{RPCA}=r_{RKPCA}=50$; ORL, $r_{NLRR}=r_{RPCA}=r_{RKPCA}=41$; COIL20, $r_{NLRR}=r_{RPCA}=200$, $r_{RKPCA}=400$; YaleB+, $r_{NLRR}=r_{RPCA}=r_{RKPCA}=342$. The parameters $q$ (shown in Algorithm \ref{alg.sc_rkpca}) for constructing the similarity graph in NLRR, RPCA, and RKPCA for the four datasets are set as $q_{MNIST}=8$, $q_{ORL}=4$, $q_{COIL20}=6$, and $q_{YaleB+}=4$. The clustering errors (averages of 10 repeated trials) defined as
\begin{equation}
\text{clustering error}=\dfrac{\#\ \text{of misclassified points}}{\text{total}\ \#\ \text{of points}}
\end{equation}
are reported in Table \ref{Tab.real2}. It is found that for the noisy images, SSC, NLRR and RKLRS have significantly higher clustering errors than RPCA and RKPCA do. The reason is that SSC, NLRR, and RKLRS can not provide effective affinity matrices if the data are heavily corrupted. RKLRS can only handle outliers and the outliers should not be caused by sparse noise. In contrast, RPCA and RKPCA are more robust to noises. RKPCA outperforms other methods in all cases. Such clustering results are consistent with the previous noise removal results.

\begin{table}[h!]
\caption{Clustering errors ($\%$) on natural images}
\centering
\begin{tabular}{lllllll}
\hline
   Dataset &  Case &        SSC &       NLRR&      RKLRS &       RPCA &      RKPCA \\
\hline
\multirow{ 3}{*}{MNIST} &    orginal &  35.2 &       29.4 &	28.3 &       30.7 &       {\bf 26.2} \\

 & pixel-noisy &       57.5 &       65.1 & 63.6 &      37.8 & {\bf 31.6} \\

 & block-noisy &       70.1 &       62.6 &  64.9 &     50.4 & {\bf 37.3} \\ \hline

\multirow{ 3}{*}{ORL} &    orginal &       22.5 &      20.75 &  21.5 &     21.5 & {\bf 19.5} \\

 & pixel-noisy &      25.75 &       36.5 &  33.25 &    22.75 & {\bf 19.5} \\

 & block-noisy &      34.25 &       58.5 &  50.5 &     25.8 & {\bf 20.75} \\ \hline

\multirow{ 3}{*}{COIL20} &    orginal & 14.29 &      17.92 &  14.52 &    16.72 &      {\bf 13.26} \\

 & pixel-noisy &      32.06 &      36.58 &   38.64 &   21.82 & {\bf 15.42} \\

 & block-noisy &      31.86 &      39.04 &   40.12 &   26.34 & {\bf 18.61} \\ \hline

\multirow{ 3}{*}{YaleB+} &    orginal &      22.26 &      17.46 &  18.02 &    18.25 & {\bf 13.63} \\

 & pixel-noisy &       65.1 &      73.02 &  75.2 &     28.5 & {\bf 26.17} \\

 & block-noisy &      40.87 &      35.12 &  41.36 &    27.92 & {\bf 24.41} \\
\hline
\end{tabular}\label{Tab.real2}
\end{table}

\subsection{Motion data}
The Hopkins-155 dataset \cite{4269999} is a benchmark for testing feature based motion segmentation algorithms. It contains video sequences along with the features extracted and tracked in all the frames. We choose three subsets, 1R2RC, 1RT2RCRT, and 2RT3RTCRT, to evaluate the proposed method RKPCA. Similar to \cite{yang2015sparse}, for each subset, we uniformly sample 6 frames and hence the feature dimension is 12, through which the formed matrix is of high-rank. Since the entry value of the three matrices is within interval $(-1,1)$, we add sparse noises randomly drawn from $\mathcal{U}(-1,1)$ to the matrices. 

We use RPCA and RKPCA to remove the noises of the matrices and then perform SSC to cluster the data for motion segmentation. The reason we perform clustering using SSC but not RPCA and RKPCA is that SSC has achieved state-of-the-art clustering accuracy on Hopkins-155 dataset. In addition, two variants of SSC including the SCHQ method proposed in \cite{7159061} and the SR-SSC proposed in \cite{SRSSC} are also compared in this study. In RKPCA, the parameter $\sigma$ of RBF kernel is determined by (\ref{Eq.kerpar}) with $\beta=0.5$. As shown in Figure \ref{Fig.Hopk_r}, the recovery error (defined by (\ref{Eq.e_rlt})) of RKPCA is often only $50\%$ of that of RPCA on the three data subsets. The clustering errors are shown in Figure \ref{Fig.Hopk_c}. In every case, the clustering error of RKPCA+SSC is considerably lower than those of SSC, RPCA+SSC, SCHQ, and SR-SSC.

\begin{figure}[h!]
\centering
\includegraphics[width=8.5cm]{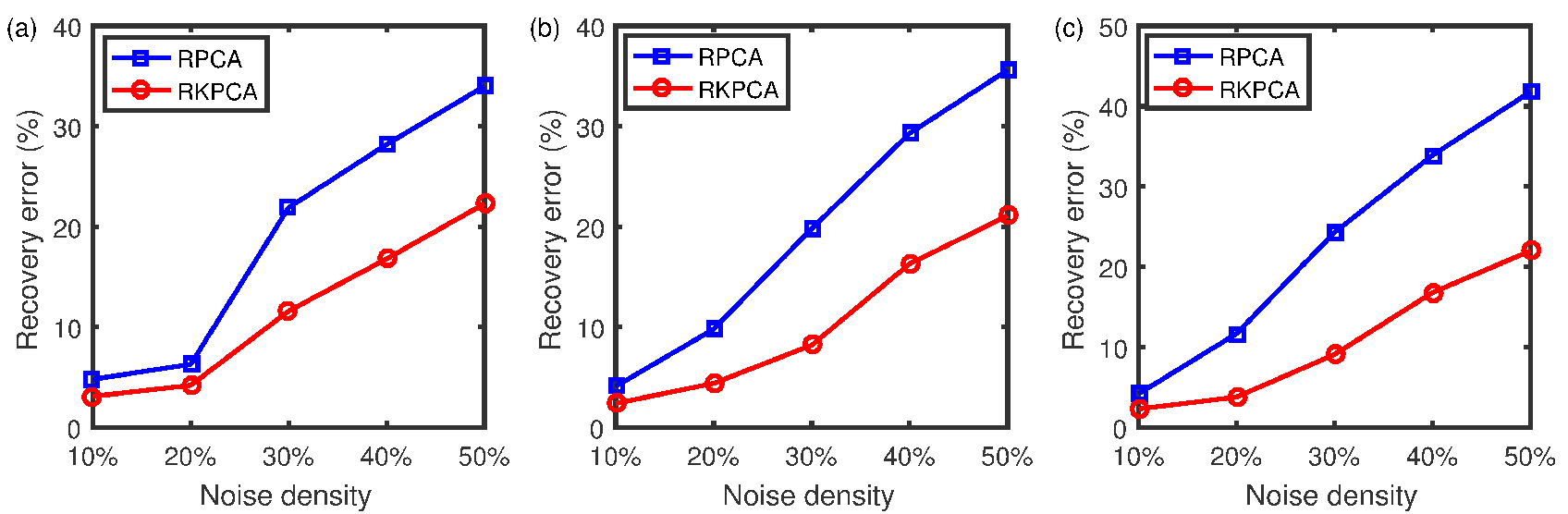}
\caption{Recovery errors on Hopkins-155 data: (a) 1R2RC; (b) 1RT2RCRT; (c) 2RT3RTCRT}\label{Fig.Hopk_r}
\end{figure}

\begin{figure}[h!]
\centering
\includegraphics[width=8.5cm]{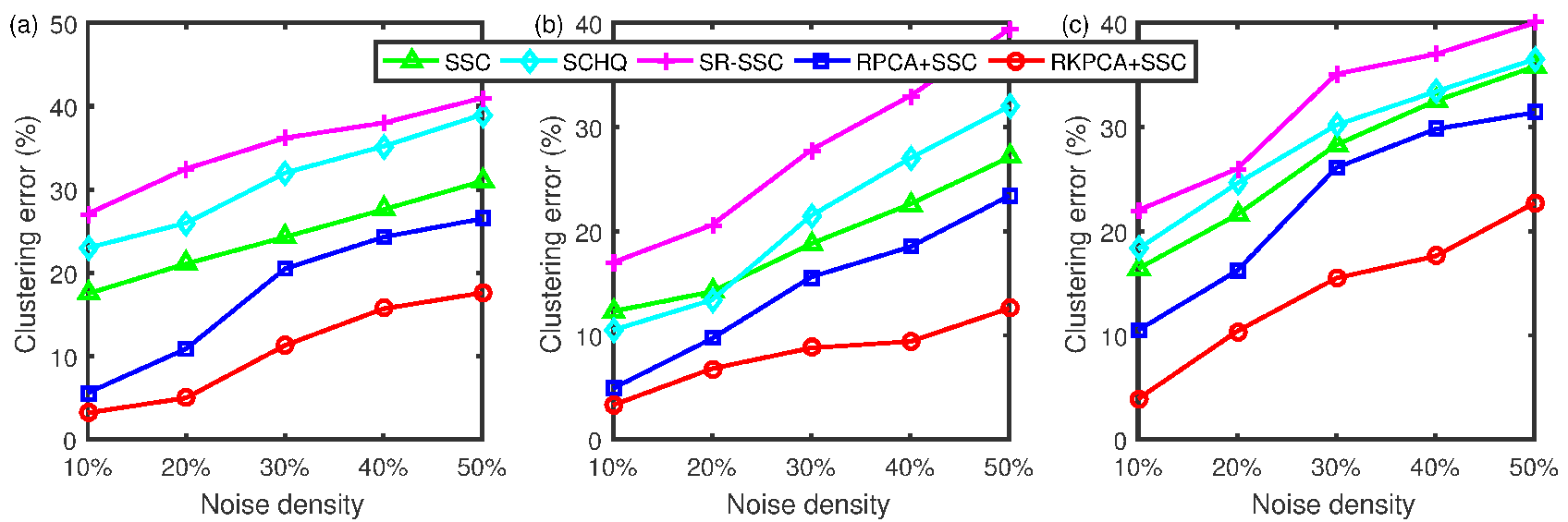}
\caption{Clustering errors on Hopkins-155 data: (a) 1R2RC; (b) 1RT2RCRT; (c) 2RT3RTCRT}\label{Fig.Hopk_c}
\end{figure}

\section{Conclusion}
Recovering corrupted high/full-rank matrix is a challenge issue to existing matrix recovering methods including RPCA. We derived and proposed the RKPCA method to solve the problem.  In this paper, we included an important theoretical proof showing the superiority of RKPCA over RPCA theoretically. This paper shows that RKPCA to date is the only unsupervised nonlinear method that exhibits strong robustness to sparse noises. We also proposed two nonconvex algorithms, ADMM+BTLS and PLM+AdSS, to solve the challenging optimization of RKPCA. The convergence proof are also included. It is interesting to note that PLM+AdSS is more computationally efficient than ADMM+BTLS. Thorough comparative studies showed that RKPCA significantly outperformed PCA, KPCA, RPCA, SSC, NLRR, and RKLRS in noise removal, subspace clustering, or/and motion segmentation.
%\subsubsection*{Acknowledgements}

%Use unnumbered third level headings for the acknowledgements.  All
%acknowledgements go at the end of the paper.

%\subsubsection*{References}
\bibliography{RKPCA}
\bibliographystyle{unsrt}

\begin{IEEEbiography}[{\includegraphics[width=1in,height=1.25in,clip,keepaspectratio]{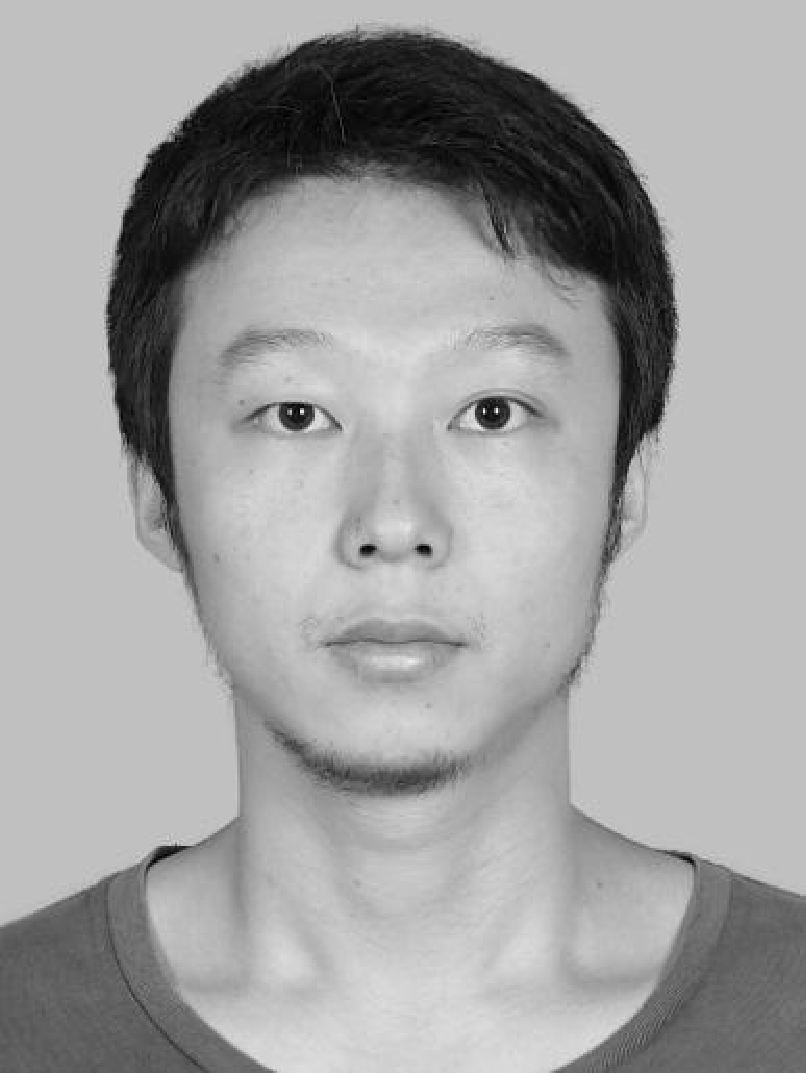}}]{Jicong Fan}
received his B.E and M.E degrees in Automation and Control Science \& Engineering, from Beijing University of Chemical Technology, Beijing, P.R., China, in 2010 and 2013, respectively. From 2013 to 2015, he was a research assistant at the University of Hong Kong and focused on signal processing and data analysis for neuroscience. Currently, he is working toward the PhD degree at the Department of Electronic Engineering, City University of Hong Kong, Kowloon, Hong Kong S.A.R. His research interests include data mining and machine learning.\\
\end{IEEEbiography}
%\vspace*{-5\baselineskip}
\begin{IEEEbiography}[{\includegraphics[width=1in,height=1.25in,clip,keepaspectratio]{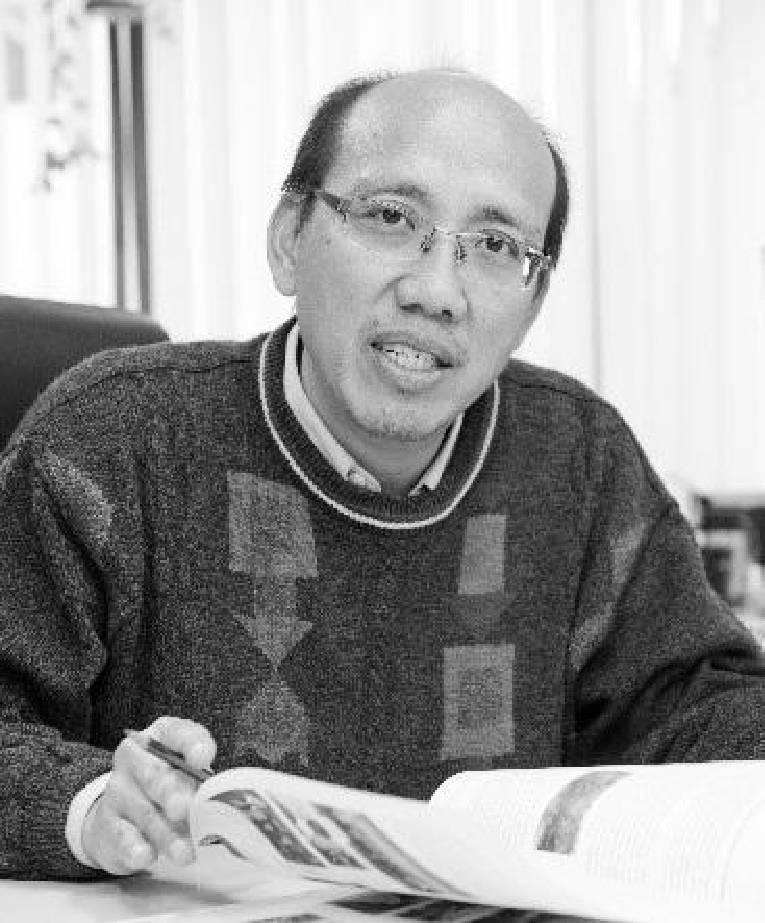}}]{Tommy W. S. Chow}
received the B.Sc. (1st Hons) degree and the Ph.D. degree from the Department of Electrical and Electronic Engineering, University of Sunderland, U.K. He is currently a Professor in the Department of Electronic Engineering at the City University of Hong Kong. His main research areas include neural networks, machine learning, pattern recognition, and fault diagnosis.  He received the Best Paper Award in 2002 IEEE Industrial Electronics Society Annual meeting in Seville, Spain. He is an author and co-author of over 170 technical Journal articles related to his research, 5 book chapters, and 1 book.  He now serves the IEEE Transactions on Industrial Informatics and Neural Processing letters as Associate editor. He is a fellow of IEEE.
\end{IEEEbiography}

\end{document}